\documentclass[10pt]{article}
\usepackage[margin=1.2in]{geometry}

\usepackage{verbatim}
\usepackage{graphicx}
\usepackage{amsmath}
\usepackage{amssymb}
\usepackage{amsthm}
\usepackage{algorithm2e}
\usepackage{color}
\usepackage{xr}
\usepackage{wrapfig}
\usepackage{bbm}
\usepackage{subcaption}
\usepackage{afterpage}
\usepackage{natbib}

\usepackage{url}

\newcommand{\mx}{\mathcal{X}}

\newcommand{\mz}{\mathcal{Z}}

\newcommand{\md}{\mathcal{D}}

\newcommand{\ceta}{c_{\eta}}

\newcommand{\reals}{\mathbb{R}}

\newcommand{\norm}[1]{\| #1 \|}

\newcommand{\wvec}[1]{\ww^{(#1)}}
\newcommand{\uvec}[1]{\uu^{(#1)}}

\newcommand{\net}{N_{W}}
\newcommand{\nett}[1]{N_{W_{#1}}}

\newcommand{\abs}[1]{\left|{#1}\right|}
\newcommand\numberthis{\addtocounter{equation}{1}\tag{\theequation}}
\DeclareMathOperator*{\argmax}{arg\,max}
\DeclareMathOperator{\sign}{sign}
\newtheorem{thm}{Theorem}[section]

\newtheorem{prop}[thm]{Proposition}
\newtheorem{lem}[thm]{Lemma}
\newtheorem{definition}[thm]{Definition}

\usepackage{boxedminipage}

\renewcommand{\xi}{{\xx}^{(m)}}

\newcommand{\prob}{\mathbb{P}}
\newcommand{\probarg}[1]{\prob\left[{#1}\right]}

\newcommand{\needcite}[1]{}

\newcommand{\be}{\begin{equation}}
\newcommand{\ee}{\end{equation}}
\newcommand{\benn}{\begin{equation*}}
\newcommand{\eenn}{\end{equation*}}
\newcommand{\bea}{\begin{eqnarray*}}
\newcommand{\eea}{\end{eqnarray*}}
\newcommand{\bean}{\begin{eqnarray}}
\newcommand{\eean}{\end{eqnarray}}

\newcommand{\ww}{\boldsymbol{w}} 
\newcommand{\xx}{\boldsymbol{x}}

\newcommand{\zz}{\boldsymbol{z}}

\newcommand{\vv}{\boldsymbol{v}}

\newcommand{\uu}{\boldsymbol{u}} 
\newcommand{\pp}{\boldsymbol{p}}

\newcommand{\agtodo}[1]{({\textcolor{red}{#1}})}

\newcommand{\ignore}[1]{}

\newcommand{\polyring}[1]{\reals\left[x_1,\ldots,x_n\right]}

{
\begin{center}
\begin{boxedminipage}{0.8\linewidth}
\begin{center}
\textbf{\texttt{#1}}
\end{center}
\rm
\begin{tabbing}
....\=...\=...\=...\=...\=  \+ \kill
} %
{\end{tabbing} 
\end{boxedminipage} \end{center} 
}

\renewcommand{\eqref}[1]{Eq.~\ref{#1}}

\newcommand{\figref}[1]{Fig.~\ref{#1}}
\newcommand{\secref}[1]{Sec.~\ref{#1}}

\newtheorem{remark}[thm]{Remark}
\title{Why do Larger Models Generalize Better? A Theoretical Perspective via the XOR Problem}

\author{
	Alon Brutzkus \\ \texttt{alonbrutzkus@mail.tau.ac.il}
	\and
	Amir Globerson \\
	\texttt{gamir@post.tau.ac.il}
}
\date{%
	The Blavatnik School of Computer Science, Tel Aviv University %
}

\begin{document}
\maketitle

\begin{abstract} 
	Empirical evidence suggests that neural networks with ReLU activations generalize better with over-parameterization. However, there is currently no theoretical analysis that explains this observation. In this work, we provide theoretical and empirical evidence that, in certain cases, overparameterized convolutional networks generalize better than small networks because of an interplay between weight clustering and feature exploration at initialization.  We demonstrate this theoretically for a 3-layer convolutional neural network with max-pooling, in a novel setting which extends the XOR problem. We show that this interplay implies that with overparamterization, gradient descent converges to global minima with better generalization performance compared to global minima of small networks. Empirically, we demonstrate these phenomena for a 3-layer convolutional neural network in the MNIST task.
\end{abstract}

\section{Introduction}
\label{sec:intro}
Most successful deep learning models use more parameters than needed to achieve zero training error. This is typically referred to as \textit{overparameterization}.
Indeed, it can be argued that overparameterization is one of the key techniques that has led to the remarkable success of neural networks. However, there is still no theoretical account for its effectiveness.
 
One very intriguing observation in this context is that overparameterized networks with ReLU activations, which are trained with gradient based methods, often exhibit better generalization error than smaller networks  \citep{neyshabur2014search,neyshabur2018towards,novak2018sensitivity}. 
In particular, it often happens that two networks, one with $N_1$ neurons and one with $N_2>N_1$ neurons achieve zero training error, but the larger network has better test error.
This somewhat counter-intuitive observation suggests that first-order methods which are trained on overparameterized networks have an \textit{inductive bias} towards solutions with better generalization performance. Understanding this inductive bias is a necessary step towards a full understanding of neural networks in practice. 

Providing theoretical guarantees for overparameterization is extremely challenging due to two main reasons. First, to show a generalization gap between smaller and larger models, one needs to prove that large networks have better sample complexity than smaller ones. However, current generalization bounds that are based on complexity measures do not offer such guarantees.\footnote{We note that better generalization \textit{upper} bounds for overparameterized networks do not prove this.} Second, analyzing
convergence of first-order methods on networks with ReLU activations is a major challenge. Indeed, there are no optimization guarantees even for simple learning tasks such as the classic two dimensional XOR problem. Given these difficulties, it is natural to analyze a simplified scenario, which ideally shares various features with real-world settings.

In this work we follow this approach and show that a possible explanation for the success of overparameterization is a combination of two effects: weight exploration and weight clustering. Weight exploration refers to the fact that larger models explore the set of possible weights more effectively since they have more neurons in each layer. Weight clustering is an effect we demonstrate here, which refers to the fact that weight vectors in the same layer tend to cluster around a small number of prototypes.  

To see \textit{informally} how these effects act in the case of overparameterization, consider a binary classification problem and a training set. The training set typically contains multiple patterns that discriminate between the two classes. The smaller network will find detectors (e.g., convolutional filters) for a subset of these patterns and reach zero training error, but not generalize because it is missing some of the patterns. This is a result of an under-exploration effect for the small net. On the other hand, the larger net has better exploration and will find more relevant detectors for classification. Furthermore, due to the clustering effect its weight vectors will be close to a small set of prototypes. Therefore the effective capacity of the overall model will be restricted, leading to good generalization.

The network we study here includes some key architectural components used in modern machine learning models. Specifically, it consists of a convolution layer with a ReLU activation function, followed by a max-pooling operation, and a fully-connected layer. This is a key component of most machine-vision models, since it can be used to detect patterns in an input image. We are also not aware of any theoretical guarantees for a network of this structure.

For this architecture, we consider the problem of detecting two dimensional binary patterns in a high dimensional input vector. The patterns we focus on are the XOR combination (i.e., $(1,1)$ or $(-1,-1)$). This problem is a high dimensional extension of the XOR problem. We refer to it as the ``XOR Detection problem (XORD). 
One advantage of this setting is that it nicely exhibits the phenomenon of overparameterization empirically, and is therefore a good test-bed for understanding overparameterization. \figref{fig:xor_over_param} shows the result of learning the XORD problem with the above network, and different number of channels. It can be seen that increasing the number of channels improves test error.\footnote{Note that a similar curve is observed when only considering zero training error, implying that smaller networks are expressive enough to fit the training data.}


\begin{figure}	
	
	\centering
	\includegraphics[width=.5\linewidth]{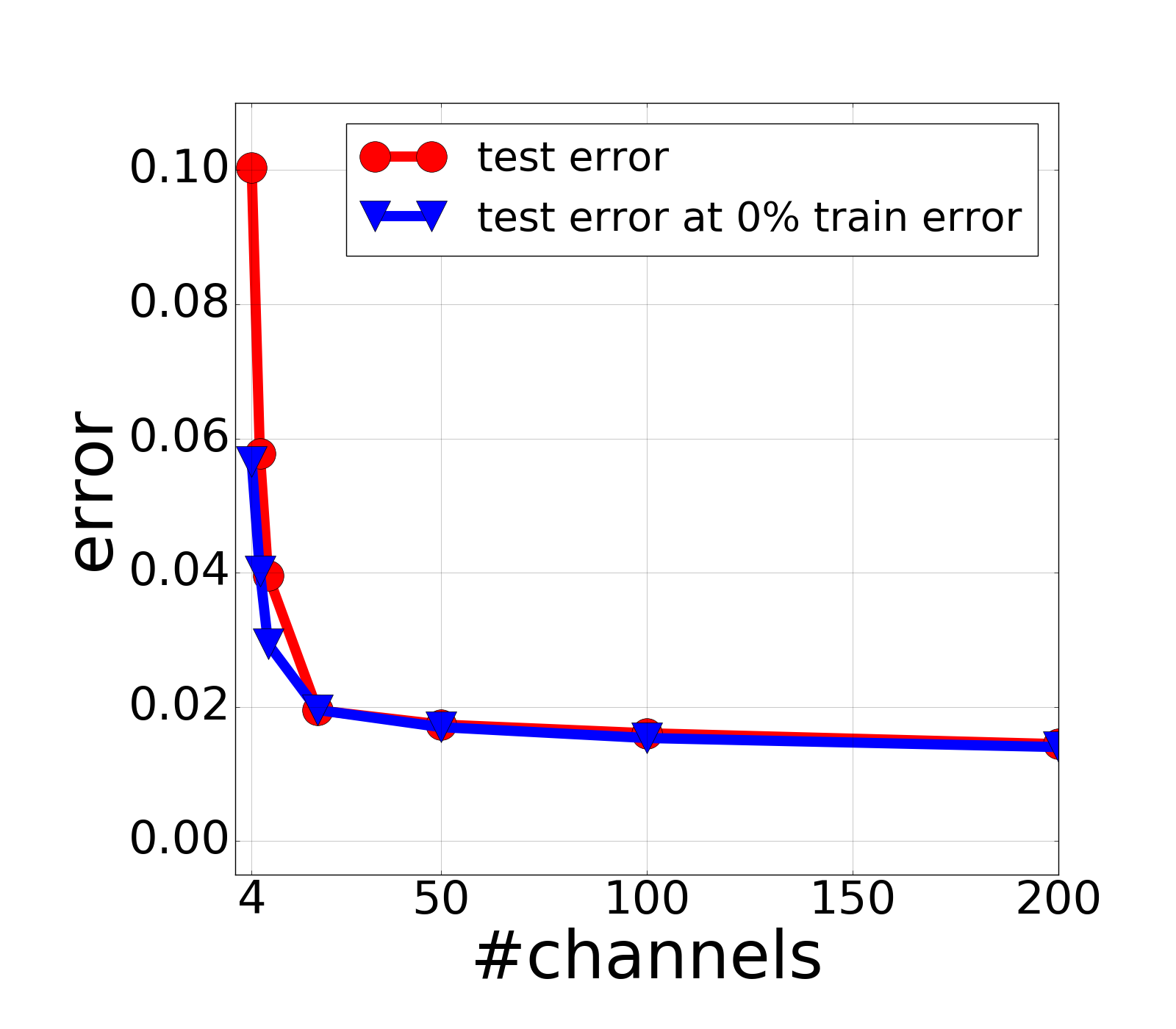}
	\caption{\small{overparameterization improves generalization in the XORD problem.  The network in \eqref{eq:xord_network} is trained on data from the XORD problem (see \secref{sec:xord_problem_formulation}). The figure shows the test error obtained
	for different number of channels $k$. The blue curve shows test error when restricting to cases where training error was zero. It can be seen that increasing the number of channels improves the generalization performance. Experimental details are provided in supplementary material.}}.
	\label{fig:xor_over_param}
\end{figure}

Motivated by these empirical observations, we present a theoretical analysis of optimization and generalization in the XORD problem. Under certain distributional assumptions, we will show that overparameterized networks enjoy a combination of better exploration of features at initialization and clustering of weights, leading to better generalization for overparameterized networks. 


\ignore{
\textbf{Contribution: } In this work we provide the first analysis which proves that larger models generalize better than smaller ones.
Concretely, we analyze the optimization and generalization of gradient descent in the XORD problem. We show that overparameterized networks enjoy a combination of better exploration of features at initialization and clustering of weights at convergence. We then prove that this implies that overparameterized networks are biased towards global minima with better generalization performance than global minima of smaller networks.
}
Importantly, we show empirically that our insights from the XORD problem transfer to other settings. In particular, we see a similar phenomenon when learning on the MNIST data, where we verify that weights are clustered at convergence and better exploration of weights for large networks.



Finally, another contribution of our work is the first proof of convergence of gradient descent in the classic XOR problem with inputs in $\{\pm1\}^2$. The proof is simple and conveys the key insights of the analysis of the  general XORD problem. See Section \ref{sec:xor} for further details.

\section{Related Work}
\label{sec:related_work}

In recent years there have been many works on theoretical aspects of deep learning. We will refer to those that are most relevant to this work. First, we note that we are not aware of any work that shows that generalization performance provably improves with over-parameterization. This distinguishes our work from all previous works. 

Several works study convolutional networks with ReLU activations and their properties \citep{du2017convolutional,du2017gradient,brutzkus2017globally}. All of these works consider convolutional networks with a single channel. Recently, there have been numerous works that provide guarantees for gradient-based methods in general settings \citep{daniely2017sgd, li2018learning,du2018gradient,du2018gradientdeep, allen2018learning}. However, their analysis holds for over-parameterized networks with an extremely large number of neurons that are not used in practice (e.g., the number of neurons is a very large polynomial of certain problem parameters). Furthermore, we consider a 3-layer convolutional network with max-pooling which is not studied in these works.

 \cite{soltanolkotabi2018theoretical}, \cite{du2018power} and \cite{li2017algorithmic} study the role of over-parameterization in the case of quadratic activation functions. \cite{brutzkus2018sgd} provide generalization guarantees for over-parameterized networks with Leaky ReLU activations on linearly separable data. 
\cite{neyshabur2018towards} prove generalization bounds for neural networks. However, these bounds are empirically vacuous for over-parameterized networks and they do not prove that networks found by optimization algorithms give low generalization bounds.

\section{Warm up: the XOR Problem}
\label{sec:xor}

We begin by studying the simplest form of our model: the classic XOR problem in two dimensions.\footnote{XOR  is a specific case of XORD in \secref{sec:xord_problem_formulation} where $d=1$.}
We will show that this problem illustrates the key phenomena that allow overparameterized networks to perform better than smaller ones. Namely, exploration at initialization and clustering during training. For the XOR problem, this will imply that overparameterized networks have better \textit{optimization} performance.  In later sections, we will show that the same phenomena occur for higher dimensions in the XORD problem and imply better \textit{generalization} of global minima for overparameterized convolutional networks.



\subsection{Problem Formulation}
\label{sec:xor_problem_formulation}

In the XOR problem, we are given a training set $S = \left\{(\xx_i,y_i)\right\}_{i=1}^4 \subseteq \{\pm1\}^2 \times \{\pm1\}^2$ consisting of points  $\xx_1 = (1,1)$, $\xx_2 = (-1,1)$, $\xx_3 = (-1,-1)$, $\xx_4 = (1,-1)$ with labels $y_1 = 1$, $y_2 = -1$, $y_3 = 1$ and $y_4 = -1$, respectively. Our goal is to learn the XOR function $f^*:\{\pm 1\}^2 \rightarrow \{\pm 1\}$, such that $f^*(\xx_i) = y_i$ for $1 \le i \le 4$, with a neural network and gradient descent.

\paragraph{\underline{Neural Architecture:}} For this task we consider the following two-layer fully connected network.

\begin{equation}
\label{eq:xor_network}
\net(\xx)=\sum_{i=1}^{k}\left[\sigma\left(\wvec{i}\cdot \xx \right) - \sigma\left(\uvec{i}\cdot \xx \right)\right]
\end{equation}
where $W \in \reals^{2k \times 2}$ is the weight matrix whose rows are the $\wvec{i}$ vectors followed by the $\uvec{i}$ vectors, and $\sigma(x)=\max\{0,x\}$ is the ReLU activation applied element-wise. We note that $f^*$ can be implemented with this network for $k=2$ and this is the minimal $k$ for which this is possible. Thus we refer to $k > 2$ as the overparameterized case.

\paragraph{\underline{Training Algorithm:}} The parameters of the network $\net(\xx)$ are learned using gradient descent on the hinge loss objective.
We use a constant learning rate $\eta \le \frac{\ceta}{k}$, where $\ceta < \frac{1}{2}$. The parameters $\net$ are initialized as IID Gaussians with zero mean and standard deviation $\sigma_g \le \frac{\ceta}{16k^{3/2}}$. We consider the hinge-loss objective:
 $$\ell(W)=\sum_{(\xx,y)\in S}{\max\{1-y\net(\xx),0\}}$$
where optimization is only over the first layer of the network. We note that for $k \ge 2$ any global minimum $W$ of $\ell$ satisfies $\ell(W) = 0$ and $\sign(\net(\xx_i)) = f^*(\xx_i)$ for $1 \le i \le 4$.

\paragraph{\underline{Notations:}} We will need the following notations. Let $W_t$ be the weight matrix at iteration $t$ of gradient descent. For $1 \leq i \leq k$, denote by $\wvec{i}_t \in \reals^2$ the $i^{th}$ weight vector at iteration $t$. Similarly  we define $\uvec{i}_t \in \reals^2$ to be the $k + i$ weight vector at iteration $t$. 
For each point $\xx_i \in S$ define the following sets of neurons:
\begin{eqnarray}
W_t^+(i) &=& \left\{j \mid \wvec{j}_t\cdot \xx_i > 0\right\} \nonumber \\
U_t^+(i) &=& \left\{j \mid \uvec{j}_t\cdot \xx_i > 0\right\} \nonumber
\end{eqnarray}
and for each iteration $t$, let $a_i(t)$ be the number of iterations $0 \le t' \le t$ such that $y_i\nett{t'}(\xx_i) < 1$.

\begin{figure*}[t]
	\begin{subfigure}{.24\textwidth}
		\centering
		\includegraphics[width=1.0\linewidth]{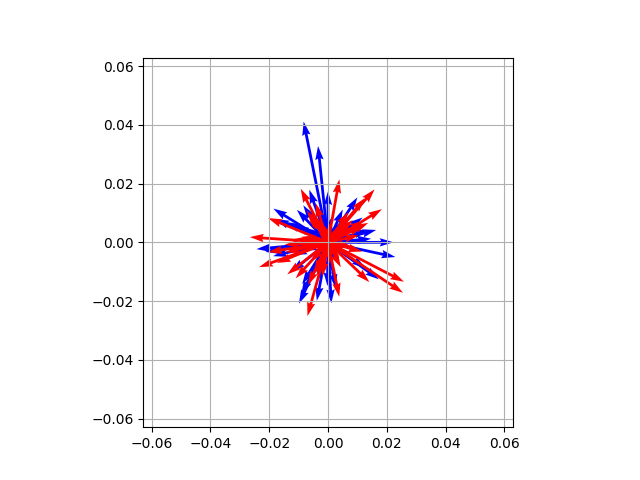}
		\caption{}
		\label{fig:exp_xor1}
	\end{subfigure}%
	\begin{subfigure}{.24\textwidth}
		\centering
		\includegraphics[width=1.0\linewidth]{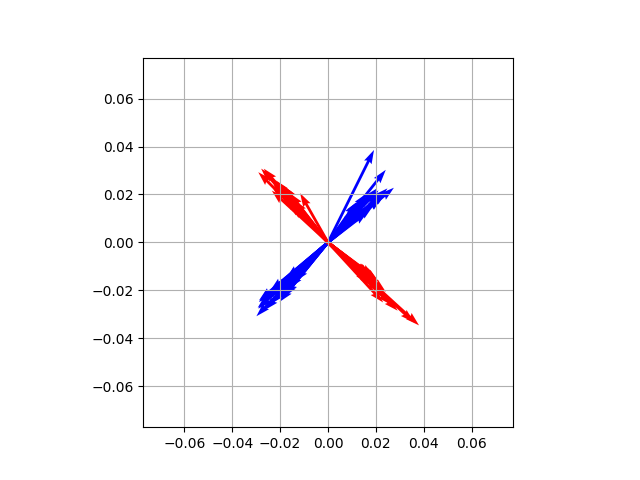}
		\caption{}
		\label{fig:exp_xor2}
	\end{subfigure}%
	\begin{subfigure}{.24\textwidth}
		\centering
		\includegraphics[width=1.0\linewidth]{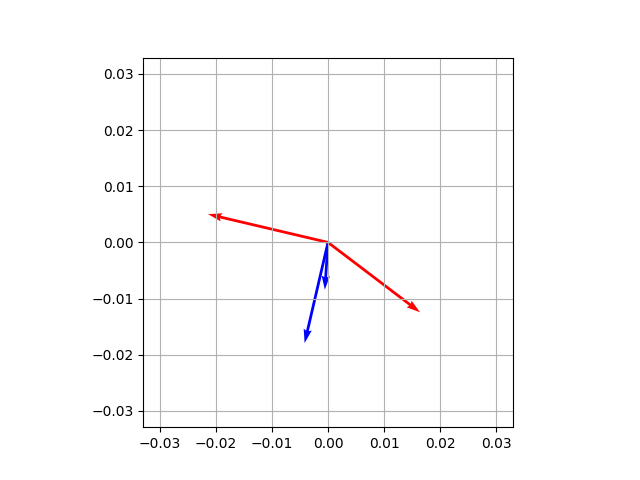}
		\caption{}
		\label{fig:exp_xor3}
	\end{subfigure}
	\begin{subfigure}{.24\textwidth}
		\centering
		\includegraphics[width=1.0\linewidth]{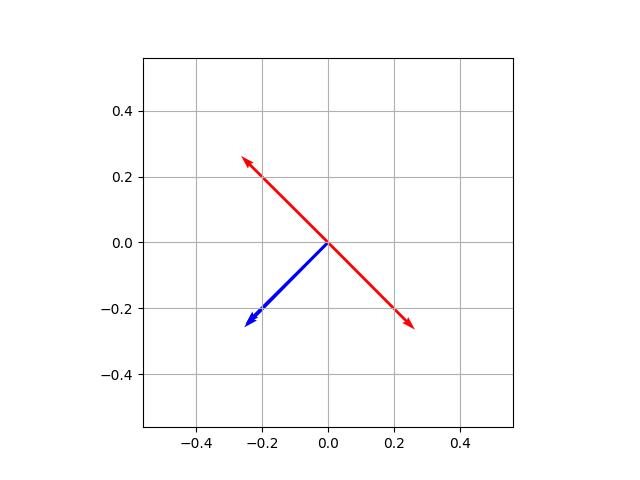}
		\caption{}
		\label{fig:exp_xor4}
	\end{subfigure}
	\caption{\small{Overparameterization and optimization in the XOR problem. The vectors in blue are the vectors $\wvec{i}_t$ and in red are the vectors $\uvec{i}_t$. (a) Exploration at initialization (t=0) for $k=50$ (Lemma \ref{lem:exploration_init}) (b) Clustering and convergence to global minimum for $k=50$ (Lemma \ref{lem:clustering_xor} and Theorem \ref{thm:xor_overparam}) (c) Non-sufficient exploration at initialization (t=0) for $k=2$ (Theorem \ref{thm:xor_local_min}). (d) Convergence to local minimum (Theorem \ref{thm:xor_local_min}).}}
	\label{fig:exp_xor}
\end{figure*}

\subsection{Over-parameterized Networks Optimize Well}
\label{sec:xor_overparam}

In this section we assume that $k > 16$. The following lemma shows that with high probability, for every training point, overparameterized networks are initialized at directions that have positive correlation with the training point. The proof uses a standard measure concentration argument. We refer to this as ``exploration'' as it lets the
optimization procedure explore these parts of weight space.
\begin{lem}
	\label{lem:exploration_init}
	\textbf{Exploration at Initialization.} With probability at least $1-8e^{-8}$, for all $1 \le i \le 4$ $$\frac{k}{2} -2\sqrt{k}\le \abs{W_0^+(i)},\abs{U_0^+(i)} \le \frac{k}{2} +2\sqrt{k}$$
\end{lem}

Next, we show an example of the weight dynamics which imply that the weights tend to cluster around a few directions. The proof uses the fact that with high probability the initial weights have small norm and proceeds by induction on $t$ to show the dynamics.
\begin{lem}
	\label{lem:clustering_xor}
	\textbf{Clustering Dynamics.} Let $i \in \{1,3\}$. With probability $\ge 1-\frac{\sqrt{2k}}{\sqrt{\pi} e^{8k}}$, for all $t \ge 0$ and $j \in W_0^+(i)$ there exists a vector $\vv_t$ such that $v_t \cdot \xx_i > 0$, $\abs{v_t \cdot \xx_2} < 2\eta$ and $\wvec{j}_t = a_i(t)\eta\xx_i + \vv_t$.
\end{lem}

The sequence $\{a_i(t)\}_{t \ge 0}$ is non-decreasing and it can be shown that $a_i(0) = 1$ with high probablity. Therefore, the above lemma shows that for all $j \in W_0^+(i)$, $\wvec{j}_t$ tends to cluster around $\xx_i$ as $t$ increases. Since with probability $1$, $W_0^+(1) \cup W_0^+(3) = [k]$, the above lemma characterizes the dynamics of all filters $\wvec{j}_t$. In the supplementary we show a similar result for the filters $\uvec{j}_t$.

By applying both of the above lemmas, it can be shown that for $k > 16$ gradient descent converges to a global minimum with high probability and that the weights are clustered at convergence.

\begin{thm}
	\textbf{Convergence and Clustering.}\label{thm:xor_overparam}
	With probability $\ge 1 - \frac{\sqrt{2k}}{\sqrt{\pi} e^{8k}} - 8e^{-8}$ 
after at most $T \le \frac{16\sqrt{k}}{\sqrt{k}-2}$ iterations, gradient descent converges to a global minimum $W_T$. Furthermore, for $i \in \{1,3\}$ and all $j \in W_0^+(i)$, the angle between  $\wvec{j}_T$ and $\xx_i$ is at most $\arccos\left(\frac{1-2\ceta}{1+\ceta}\right)$. A similar result holds for $\uvec{j}_T$.
\end{thm}

\subsection{Small Network Fail to Optimize}
\label{sec:xor_small}
In contrast to the case of large $k$, we show that for $k=2$, the initialization does not explore all directions, leading to convergence
to a suboptimal solution.
\begin{thm}
	\label{thm:xor_local_min}
	\textbf{Insufficient Exploration at Initialization.} With probability at least $0.75$, there exists $i \in \{1,3\}$ such that  $W_0^+(i) = \emptyset$ or $i \in \{2,4\}$ such that $U_0^+(i) = \emptyset$. As a result, with probability $\ge 0.75$, gradient descent converges to a model which errs on at least one input pattern.
\end{thm}

\subsection{Experiments}
\label{sec:xor_experiments}
In this section we empirically demonstrate the theoretical results. We implemented the learning setting described in \secref{sec:xor_problem_formulation} and conducted two experiments: one with $k=50$ and one with $k=2$ We note that for $k=2$ the XOR function $f^*$ can be realized by the network in \eqref{eq:xor_network}.   
Figure \ref{fig:exp_xor} shows the results. It can be seen that our theory nicely predicts the behavior of gradient descent. For $k=50$ we see the effect of exploration at initialization and clustering which imply convergence to global minimum. In contrast, the small network does not explore all directions at initialization and therefore converges to a local minimum. This is despite the fact that it has sufficient expressive power to implement $f^*$.

\section{The XORD Problem}
\label{sec:xord_problem_formulation}
In the previous section we analyzed the XOR problem, showing that using a large number of channels allows gradient descent to learn the XOR function. This allowed us to understand the effect of overparameterization on optimization. However, it did not let us study generalization because in the learning setting all four examples were given, so that any model with zero training error also had zero test error.

In order to study the effect of overparameterization on generalization we consider a more general setting, which we refer to as the XOR Detection problem (XORD). As can be seen in \figref{fig:xor_over_param}, in the XORD problem large networks generalize better than smaller ones. This is despite the fact that small networks can reach zero training error. Our goal is to understand this phenomenon from a theoretical persepective.


In this section, we define the XORD problem. We begin with some notations and definitions. We consider a classification problem in the space $ \{\pm1\}^{2d}$, for $d \ge 1$. Given a vector $\xx \in \{\pm1\}^{2d}$, we consider its partition into $d$ sets of two coordinates as follows $\xx = (\xx_1,...,\xx_d)$ where $\xx_i \in  \{\pm1\}^2$. We refer to each such  $\xx_i$ as a \textit{pattern} in $\xx$. 
\paragraph{\underline{Neural Architecture:}} We consider learning with the following three-layer neural net model. The first layer is a convolutional layer with non-overlapping filters and multiple channels, the second layer is max pooling and the third layer is a fully connected layer with $2k$ hidden neurons and weights fixed to values $\pm 1$. Formally, for an input $\xx=(\xx_1,...,\xx_d) \in \reals^{2d}$ where $\xx_i \in \reals^2$, the output of the network is denoted by $\net(\xx)$ and is given by:
\begin{equation}
\label{eq:xord_network}
\sum_{i=1}^{k}\Big[\max_j\left\{\sigma\left(\wvec{i}\cdot \xx_j \right)\right\} - \max_j\left\{\sigma\left(\uvec{i}\cdot \xx_j \right)\right\}\Big] 
\end{equation}
where notation is as in the XOR problem.
\begin{remark}
	\label{rem:expr}
	Because there are only $4$ different patterns, the network is limited in terms of the number of rules it can implement. Specifically, it is easy to show that its VC dimension is at most $15$ (see supplementary material). Despite this limited expressive power, there is a generalization gap between small and large networks in this setting, as can be seen in \figref{fig:xor_over_param}, and in our analysis below.
\end{remark} 
    
\paragraph{\underline{Data Generating Distribution:}} Next we define the classification rule we will focus on. 
Define the four two-dimensional binary patterns $\pp_1=(1,1),  \pp_2=(1,-1), \pp_3=(-1,-1), \pp_4=(-1,1)$.
 Define $P_{pos} = \{\pp_1,\pp_3\}$ to be the set of positive patterns and $P_{neg} = \{\pp_2,\pp_4\}$ to be the set of negative patterns. Define the classification rule:
 \be
 f^*(\xx) = 
 \left\{
\begin{array}{cc}
1 & \exists i\in\{1,\ldots,d\} : \xx_i \in  P_{pos}\\
-1 & \mbox{otherwise}
\end{array}
 \right.
 \ee
 Namely, $f^*$ detects whether a positive pattern  appears in the input. For $d=1$, $f^*$ is the XOR classifier in \secref{sec:xor}.


Let $\md$ be a distribution over $\mx \times \{\pm1\}$ such that for all $(\xx,y) \sim \md$ we have $y = f^*(\xx)$. We say that a point $(\xx,y)$ is positive if $y = 1$ and negative otherwise. Let $\md_+$ be the marginal distribution over $\{\pm1\}^{2d}$ of positive points and $\md_-$ be the marginal distribution of negative points.

For each point $\xx \in \{\pm1\}^{2d}$, define $P_{\xx}$ to be the set of unique two-dimensional patterns that the point $\xx$ contains, namely  $P_{\xx} = \{i \mid \exists j, \xx_j = \pp_i \}$. In the following definition we introduce the notion of \textit{diverse} points, which will play a key role in our analysis.
\begin{definition}[\bf{Diverse Points}]
\label{def:diverse}
We say that a positive point $(\xx,1)$ is diverse if $P_{\xx} = \{1,2,3,4\}$.\footnote{This definition only holds in the case that $d \ge 4$.} We say that a negative point $(\xx,-1)$ is \textit{diverse} if $P_{\xx} = \{2,4\}$. 
\end{definition}

 For $\phi \in \{-,+\}$ define $p_{\phi}$ to be the probability that $\xx$ is diverse with respect to $\md_{\phi}$. For example, if both $D_+$ and $D_-$ are uniform, then by the inclusion-exclusion principle it follows that $p_+ = 1 - \frac{4\cdot3^d - 6\cdot 2^d + 4}{4^d}$ and $p_- = 1 - \frac{1}{2^{d-1}}$.



\paragraph{\underline{Learning Setup:}}
Our analysis will focus on the problem of learning $f^*$ from training data with the three layer neural net model in \eqref{eq:xord_network}. The learning algorithm will be gradient descent, randomly initialized.
 As in any learning task in practice, $f^*$ is unknown to the training algorithm. Our goal is to analyze the performance of gradient descent when given data that is labeled with $f^*$. We assume that we are given a training set $S=S_+ \cup S_- \subseteq \{\pm1\}^{2d} \times \{\pm1\}^2$ where $S_+$ consists of $m$ IID points drawn from $\md_+$ and $S_-$ consists of $m$ IID points drawn from $\md_-$.\footnote{For simplicity, we consider this setting of equal number of positive and negative points in the training set.}

Importantly, we note that the function $f^*$ can be realized by the above network with $k=2$. Indeed, the network $N_W$ with $\wvec{1} = 3\pp_1$, $\wvec{2} = 3\pp_3$, $\uvec{1} = \pp_2$, $\uvec{2} = \pp_4$ satisfies $\sign \left(\net(\xx)\right) = f^*(\xx)$ for all $\xx \in \{\pm 1\}^{2d}$. It can be seen that for $k=1$, $f^*$ cannot be realized. Therefore, any $k>2$ is an overparameterized setting.

\paragraph{\underline{Training Algorithm:}}
We will use gradient descent to optimize the following hinge-loss function.
\begin{align*}
\label{eq:loss_function}
\ell(W)&=\frac{1}{m}\sum_{(\xx_i,y_i) \in S_+:y_i = 1}\max\{\gamma-\net(\xx_i),0\} \\ &+ \frac{1}{m}\sum_{(\xx_i,y_i) \in S_-:y_i = -1}\max\{1+\net(\xx_i),0\} \numberthis
\end{align*}
for $\gamma \ge 1$.\footnote{In practice it is common to set $\gamma$ to $1$. In our analysis we will need $\gamma \ge 8$ to guarantee generalization. In the supplementary material we show empirically, that for this task, setting $\gamma$ to be larger than $1$ results in better test performance than setting $\gamma = 1$.} We assume that gradient descent runs with a constant learning rate $\eta$ and the weights are randomly initiliazed with IID Gaussian weights with mean $0$ and standard deviation $\sigma_g$. Furthermore, only the weights of the first layer, the convolutional filters, are trained.\footnote{Note that \cite{hoffer2018fix} show that fixing the last layer to $\pm1$ does not degrade performance in various tasks. This assumption also appeared in \citep{brutzkus2018sgd,li2017convergence}.} As in Section \ref{sec:xor}, we will use the notations $W_t$, $\wvec{i}_t$, $\uvec{i}_t$ for the weights at iteration $t$ of gradient descent. At each iteration (starting from $t=0$), gradient descent performs the update $W_{t+1} = W_t - \eta \frac{\partial \ell}{\partial W}\left(W_t\right)$.

\begin{figure*}[t]
	\begin{subfigure}{.24\textwidth}
		\centering
		\includegraphics[width=1.0\linewidth]{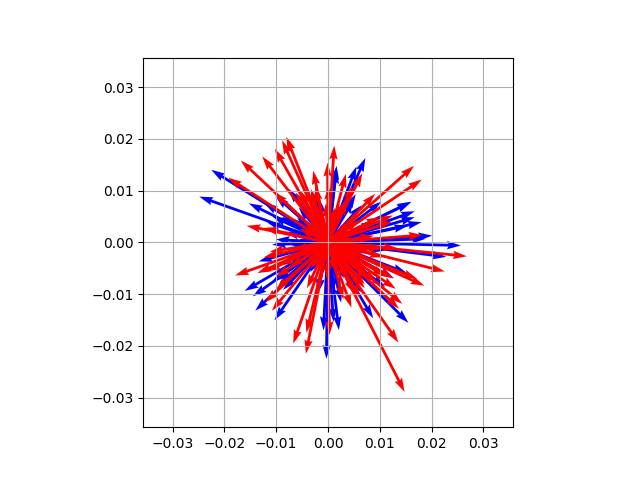}
		\caption{}
		\label{fig:exp_xord1}
	\end{subfigure}%
	\begin{subfigure}{.24\textwidth}
		\centering
		\includegraphics[width=1.0\linewidth]{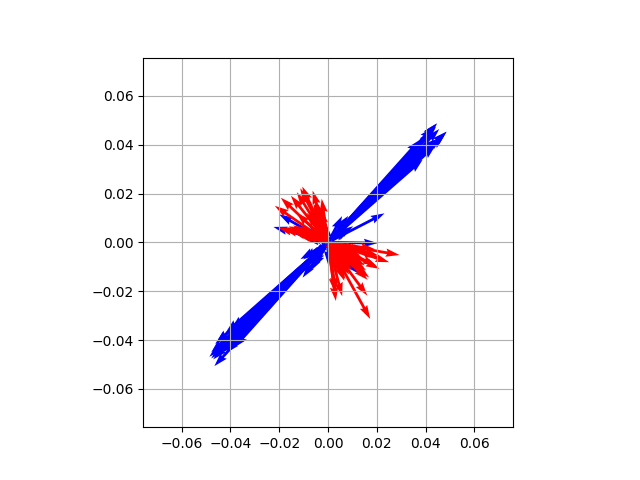}
		\caption{}
		\label{fig:exp_xord2}
	\end{subfigure}%
	\begin{subfigure}{.24\textwidth}
		\centering
		\includegraphics[width=1.0\linewidth]{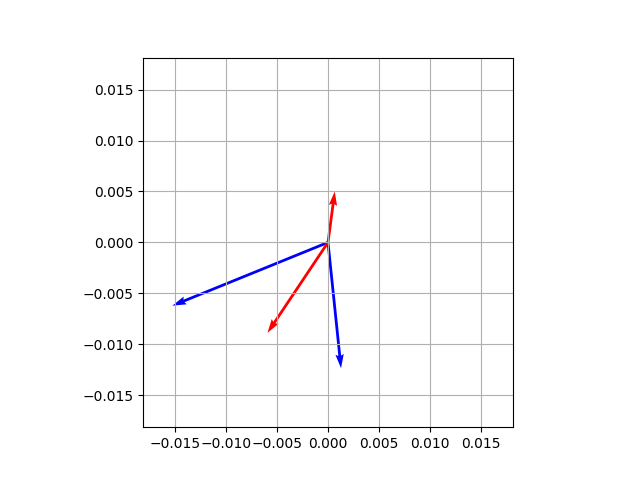}
		\caption{}
		\label{fig:exp_xord3}
	\end{subfigure}
	\begin{subfigure}{.24\textwidth}
		\centering
		\includegraphics[width=1.0\linewidth]{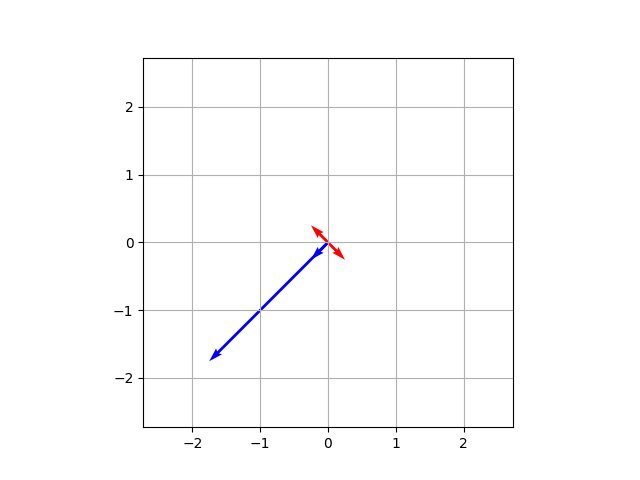}
		\caption{}
		\label{fig:exp_xord4}
	\end{subfigure}
	\caption{\small{Overparameterization and generalization in the XORD problem. The vectors in blue are the vectors $\wvec{i}_t$ and in red are the vectors $\uvec{i}_t$. (a) Exploration at initialization (t=0) for $k=100$ (b) Clustering and convergence to global minimum that recovers $f^*$ for $k=100$ (c) Non-sufficient exploration at initialization (t=0) for $k=2$. (d) Convergence to global minimum with non-zero test error for $k=2$.}}
	\label{fig:exp_xord}
\end{figure*}

\section{XORD on Decoy Sets}
\label{sec:xord_experiments}
In \figref{fig:xor_over_param} we showed that the XORD problem exhibits better generalization 
for overparameterized models. Here we will empirically show how this comes about due to the effects of clustering and exploration. We compare two networks as in \secref{sec:xord_problem_formulation}. The first has $k=2$  (i.e., four hidden neurons) and the second has $k=100$. As mentioned earlier, both these nets can achieve zero test error on the XORD problem. 

We consider a \textit{diverse} training set, namely, one which contains only diverse points. The set has 6 positive diverse points and 6 negative diverse points. Each positive point contains all the patterns $\{\pp_1,\pp_2,\pp_3,\pp_4\}$ and each negative point contains all the patterns $\{\pp_2,\pp_4\}$. Note that in order to achieve zero training error on this set, a network needs only to detect \textit{at least} one of the patterns $\pp_1$ or $\pp_3$, and \textit{at least} one of the patterns $\pp_2$ or $\pp_4$. For example, a network with $k=2$ and filters  $\wvec{1} = \wvec{2} = 3\pp_1$, $\uvec{1} = \uvec{2} = \pp_2$, has zero train loss. However, this network will not generalize to non-diverse points, where only a subset of the patterns appear. Thus we refer to it as a ``decoy'' training set.

\figref{fig:exp_xord} shows the results of training the $k=2$ and $k=100$ networks on the decoy training set. Both networks reach zero training error. However, the larger network learns the XORD function exactly, whereas the smaller network does not, and will therefore misclassify certain data points. As \figref{fig:exp_xord} clearly shows, the reason for the failure of the smaller network is that at initialization there is insufficient exploration of weight space. On the other hand, the larger network both explores well at initialization, and converges to clustered weights corresponding to all relevant patterns.

The above observations are for a training set that contains only diverse points. However, there are other decoy training sets which also contain non-diverse points (see supplementary for an example). We also note that in the experiments in \figref{fig:xor_over_param}, we trained gradient descent on various training sets which do not contain only diverse points. The generalization gap that we observe for 0 training error solutions, suggests the existence of other decoy training sets.

\section{XORD Theoretical Analysis}
\label{sec:main_results}
In \secref{sec:xord_experiments} we saw a case where overparameterized networks generalize better than smaller ones. This was due to the fact that the training set was a ``decoy'' in the sense that it could be explained by a subset of the discriminative patterns. Due to the under-exploration of weights in the smaller model this led to zero training error but non-zero test error. 

We proceed to formulate this intuition. Our theoretical results will show that for diverse training sets, networks with $k\geq 120$ will converge with high probability to a solution with zero training error that recovers $f^*$ (\secref{sec:xord_overparam}). On the other hand,
networks with $k=2$ will converge with constant probability to zero training error solutions which do not recover $f^*$ (\secref{sec:xord_small}). Finally, we show that in a PAC setting these results imply a sample complexity gap between large and small networks (\secref{sec:generalization_gap}).

\ignore{
In the case that the data distribution has sufficiently large probability for diverse points, e.g., $p_+,p_- \ge 0.9$, we will show that the latter results imply a sample complexity gap between large and small networks.
}

We assume that the training set consists of $m$ positive diverse points and $m$ negative diverse points. For the analysis, without loss of generality, we can assume that the training set consists of one positive diverse point $\xx^+$ and one negative diverse point $\xx^-$. This follows since the network and its gradient have the same value for two different positive diverse points and two different negative diverse points. Therefore, this holds for the loss function in \eqref{eq:loss_function} as well.

For the analysis, we need a few more definitions. Define the following sets for each $1 \le i \le 4$:
\begin{align*}
\label{eq:sets}
W_t^+(i) = \left\{j \mid \argmax_{1 \le l \le 4} \wvec{j}_t\cdot \pp_l = i \right\} \\
U_t^+(i) = \left\{j \mid \argmax_{1 \le l \le 4} \uvec{j}_t\cdot \pp_l = i \right\} \\
\end{align*}
For each set of binary patterns $A \subseteq \{\pm1\}^2$ define $p_A$ to be the probability to sample a point $\xx$ such that $P_{\xx} = A$. Let $A_1 = \{2\}$, $A_2 = \{4\}$, $A_3 = \{2,4,1\}$ and $A_4 = \{2,4,3\}$. The following quantity will be useful in our analysis:
\begin{equation}
\label{eq:pstar}
p^* = \min_{1\le i\le 4}{p_{A_i}}
\end{equation}

Finally, we let $a^+(t)$ be the number of iterations $0 \le t' \le t$ such that $\nett{t'}(\xx^+) < \gamma$ and $c \le 10^{-10}$ be a negligible constant.
 
\subsection{Overparameterized Network}
\label{sec:xord_overparam}

As in \secref{sec:xor_overparam}, we will show that both exploration at initialization and clustering will imply good performance of overparameterized networks. Concretely, they will imply convergence to a global minimum that recovers $f^*$. However, the analysis in XORD is significantly more involved. 

We assume that $k \ge 120$ and gradient descent runs with parameters $\eta = \frac{\ceta}{k}$ where $\ceta \le \frac{1}{410}$, $\sigma_g \le \frac{\ceta}{16k^{\frac{3}{2}}}$ and $\gamma \ge 8$.

In the analysis there are several instances of exploration and clustering effects. Due to space limitations, here we will show one such instance. In the following lemma we show an example of exploration at initialization. The proof is a direct application of a concentration bound.

\begin{lem}
	\label{lem:init_num_disj_xord}
	\textbf{Exploration.} With probability at least $1-4e^{-8}$, it holds that $\abs{\abs{W_0^+(1) \cup W_0^+(3)} - \frac{k}{2}}\le 2\sqrt{k}$.
\end{lem}

Next, we characterize the dynamics of filters in $W_0^+(1) \cup W_0^+(3)$ for all $t$.

\begin{lem}
	\label{lem:clustering_xord}
	\textbf{Clustering Dynamics.} Let $i \in \{1,3\}$. With probability $\ge 1-\frac{\sqrt{2k}}{\sqrt{\pi} e^{8k}}$, for all $t \ge 0$ and $j \in W_0^+(i)$ there exists a vector $\vv_t$ such that $v_t \cdot \pp_i > 0$, $\abs{v_t \cdot \pp_2} < 2\eta$ and $\wvec{j}_t = a^+(t)\eta\pp_i + \vv_t$.
\end{lem}

We note that $a^+(t)$ is a non-decreasing sequence such that $a^+(0) = 1$ with high probability. Therefore, the above lemma suggests that the weights in $W_0^+(1) \cup W_0^+(3)$ tend to get clustered as $t$ increases.  

By combining Lemma \ref{lem:init_num_disj_xord}, Lemma \ref{lem:clustering_xord} and other similar lemmas given in the supplementary (for  other sets $W_0^+(i), U_0^+(i)$), the following convergence theorem can be shown. The proof consists of a careful and lengthy analysis of the dynamics of gradient descent and is given in the supplementary. 

\begin{thm}
	\label{thm:main}
	 With probability at least $\left(1-c - 16e^{-8}\right)$ after running gradient descent for $T \ge \frac{28(\gamma + 1 + 8\ceta)}{\ceta}$ iterations, it converges to a global minimum which satisfies $\sign\left(\nett{T}(\xx)\right) = f^*(\xx)$ for all $\xx \in \{\pm1\}^{2d}$. Furthermore, for $i \in \{1,3\}$ and all $j \in W_0^+(i)$, the angle between  $\wvec{j}_T$ and $\pp_i$ is at most $\arccos\left(\frac{\gamma - 1-2\ceta}{\gamma-1+\ceta}\right)$. \footnote{We do not provide clustering guarantees at global minimum for other filters. However, we do characterize their dynamics similar to Lemma \ref{lem:clustering_xord}.}
\end{thm}

This result shows if the training set consists only of diverse points, then with high probability over the initialization, overparameterized networks converge to a global minimum which realizes $f^*$ in a constant number of iterations.

\subsection{Small Network \label{sec:xord_small}}
Next we consider the case of the small network $k=2$, and show that it has inferior generalization due to under-exploration. We assume that gradient descent runs with parameters values of $\eta$, $\sigma_g$ and $\gamma$ which are similar to the previous section but in a slightly broader set of values (see supplementary for details).
The main result of this section shows that with constant probability, gradient descent converges to a global minimum that does not recover $f^*$.

\begin{thm}
	\label{thm:lower_bound_specific}
	With probability at least $\left(1-c\right)\frac{33}{48}$, gradient descent converges to a global minimum that does not recover $f^*$. Furthermore, there exists $1 \le i \le 4$ such that the global minimum misclassifies all points $\xx$ such that $P_{\xx} = A_i$.
\end{thm}

The proof follows due to an \textit{under-exploration} effect. Concretely, let $\wvec{1}_T$, $\wvec{2}_T$, $\uvec{1}_T$ and $\uvec{2}_T$ be the filters of the network at the iteration $T$ in which gradient descent converges to a global minimum (convergence occurs with high constant probability). The proof shows that gradient descent will not learn $f^*$ if one of the following conditions is met: a) $W_T^+(1) = \emptyset$. b) $W_T^+(3) = \emptyset$. c) $\uvec{1}_T \cdot \pp_2 > 0$ and $\uvec{2}_T \cdot \pp_2 > 0$. d) $\uvec{1}_T \cdot \pp_4 > 0$ and $\uvec{2}_T \cdot \pp_4 > 0$.
Then by using a symmetry argument which is based on the symmetry of the initialization and the training data it can be shown that one of the above conditions is met with high constant probability.

\begin{figure*}[t]
	\begin{subfigure}{.5\textwidth}
		\centering
		\includegraphics[width=0.7\linewidth]{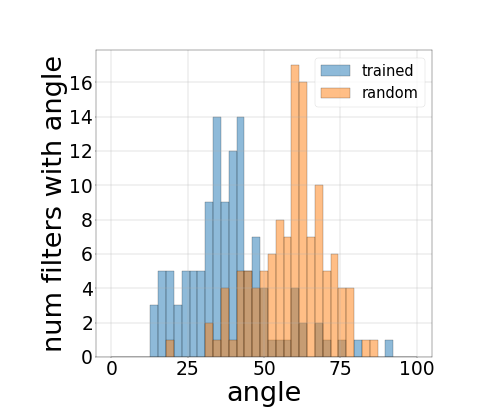}
		\caption{}
		\label{fig:exp_mnist1}
	\end{subfigure}%
	\begin{subfigure}{.5\textwidth}
		\centering
		\includegraphics[width=0.7\linewidth]{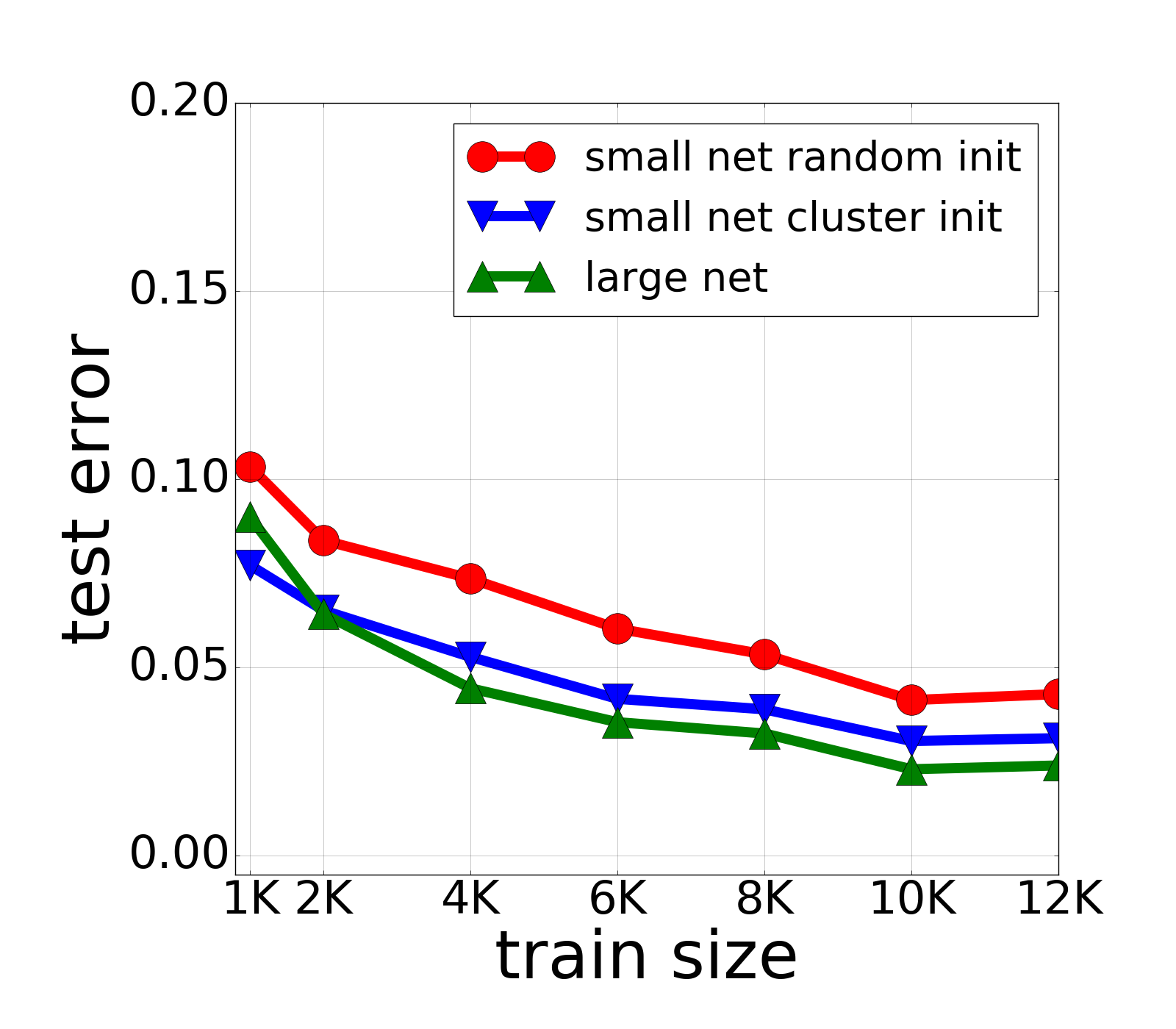}
		\caption{}
		\label{fig:exp_mnist2}
	\end{subfigure}%
	\caption{\small{Clustering and Exploration in MNIST (a) Distribution of angle to closest center in trained and random networks. (b) The plot shows the test error of the small network (4 channels) with standard training (red), the small network that uses clusters from the large network (blue), and the large network (120 channels) with standard training (green). It can be seen that the large network is effectively compressed without losing much accuracy.}}
	\label{fig:exp_mnist}
\end{figure*}

\subsection {A Sample Complexity Gap}
\label{sec:generalization_gap}
In the previous analysis we assumed that the training set was diverse. Here we consider the standard PAC setting of a distribution over inputs, and show that indeed overparameterized models enjoy better generalization. Recall that the sample complexity $m(\epsilon, \delta)$ of a learning algorithm is the minimal number of samples required for learning a model with test error at most $\epsilon$ with confidence greater than $1-\delta$ \citep{shalev2014understanding}.
\ignore{
generalization gap between overparameterized networks and networks with $k=2$. 
Define the generalization error to be the difference between the 0-1 test error and the 0-1 training error. For any $\epsilon$, $\delta$ and training algorithm let $m(\epsilon, \delta)$ be the sample complexity of a training algorithm, namely, the number of minimal samples the algorithm needs to get at most $\epsilon$ generalization error with probability at least $1-\delta$. We consider running gradient descent in two cases, when $k \ge 120$ and $k=2$ in the settings of \secref{sec:xord_overparam} and \secref{sec:xord_small}, respectively.
}

We are interested in the sample complexity of learning with $k\geq120$ and $k=2$. Denote these two functions by $m_1(\epsilon,\delta)$ and $m_2(\epsilon,\delta)$. 
The following result states that there is a gap between the sample complexity of the two models, where the larger model in fact enjoys better complexity.
\begin{thm}
	\label{thm:generalization_gap}
Let $\md$ be a distribution with paramaters $p_+$, $p_-$ and $p^*$ (see \eqref{eq:pstar}). Let $\delta \ge 1-p_+p_-(1-c-16e^{-8})$ and $0 \le \epsilon < p^*$. Then $m_1(\epsilon, \delta) \le 2$ whereas $m_2(\epsilon, \delta) \ge \frac{2\log\left(\frac{48\delta}{33(1-c)}\right)}{\log(p_+p_-)}$. \footnote{We note that this generalization gap holds for global minima (0 train error). Therefore, the theorem can be read as follows. For $k \ge 120$, given 2 samples, with probability at least $1-\delta$, gradient descent converges to a global minimum with at most $\epsilon$ test error. On the other hand, for $k=2$ and given number of samples less than $\frac{2\log\left(\frac{48\delta}{33(1-c)}\right)}{\log(p_+p_-)}$, with probability greater than $\delta$, gradient descent converges to a global minimum with error greater than $\epsilon$.}
\end{thm}
The proof (see supplementary material) follows from Theorem \ref{thm:main} and Theorem \ref{thm:lower_bound_specific} and the fact that the probability to sample a training set with only diverse points is $(p_+p_-)^m$. 

We will illustrate the guarantee of Theorem \ref{thm:generalization_gap} with several numerical examples. Assume that for the distribution $\md$, the probability to sample a positive point is $\frac{1}{2}$ and $p^* = \min\left\{\frac{1-p_+}{4}, \frac{1-p_-}{4}\right\}$ (it is easy to construct such distributions). First, consider the case $p_+ = p_- = 0.98$ and $\delta = 1 - 0.98^2(1-c-16e^{-8}) \le 0.05$. Here we get 
that for any $0 \le \epsilon < 0.005$, $m_1(\epsilon, \delta) \le 2$ whereas $m_2(\epsilon, \delta) \ge 129$. Next, consider the case where $p_+=p_- = 0.92$.  It follows that for $\delta = 0.16$ and any $0 \le \epsilon < 0.02$ it holds that $m_1(\epsilon,\delta) \le 2$ and $m_2(\epsilon,\delta) \ge 17$. In contrast, for sufficiently small $p_+$ and $p_-$, e.g., in which $p_+, p_- \le 0.7$, our bound does not guarantee a generalization gap.


\section{Experiments on MNIST}
\label{sec:experiments}
We next demonstrate how our theoretical insights from the XORD problem are also manifest when learning a neural net on the MNIST dataset.
The network we use for learning is quite similar to the one use for XORD. It is a three layer network: the first layer is a convolution with $3\times 3$ filters and multiple channels (we vary the number of channels), followed by $2\times 2$ max pooling and then a fully connected layer. We use Adam \citep{kingma2014adam} for optimization. In the supplementary we show empirical results for other filter sizes. Further details of the experiments are given there. Below we show how our two main theoretical insights for XORD are clearly exhibited in the MNIST data.

We first check the clustering observation. Namely, that optimization tends
to converge to clusters of similar filters. We train the three layer network described above with 120 channels on 6000 randomly sampled MNIST images. Then, we normalize each filter of the trained network to have unit norm. We then cluster all 120 9-dimensional vectors using kmeans to four clusters. Finally, for each filter we calculate its angle with its closest cluster center. In the second experiment we perform exactly the same procedure, but with a network with randomly initialized weights. 

\figref{fig:exp_mnist1} shows the results for this experiment. It can be clearly seen that in the trained network, most of the 9-dimensional filters have a relatively small angle with their closest center. Furthermore, the distributions of angles to closest center are significantly different in the case of trained and random networks.  This suggests that there is an inductive bias towards solutions with clustered weights, as predicted by the theory.

We next explore the effect of exploration. Namely, to what degree do larger 
models explore useful regions in weight space. The observation in our theoretical analysis is that both small and large networks can find weights that arrive at zero training error. But large networks will find a wider variety of weights, which will also generalize better. 

Here we propose to test this via the following setup: first train a large network. Then cluster its weights into $k$ clusters and use the centers to initialize a smaller network with $k$ filters. If these $k$ filters generalize better than $k$ filters learned from random initialization, this would suggest that the larger network indeed explored weight space more effectively.

To apply this idea to MNIST, We trained an ``over-parameterized'' 3-layer network with 120 channels. We clustered its filters with k-means into 4 clusters and used the cluster centers as initialization for a small network with 4 channels. Then we trained only the fully connected layer and the bias of the first layer in the small network. In \figref{fig:exp_mnist2} we show that for various training set sizes, the performance of the small network improves with the new initialization and nearly matches the performance of the over-parameterized network. This suggests that the large network explored  better features in the convolutional layer than the smaller one.

\section{Conclusions}
\label{sec:conclusion}
In this paper we consider a simplified learning task on binary vectors to study generalization of overparameterized networks. In this setting, we prove that clustering of weights and exploration of the weight space, imply better generalization performance for overparameterized networks. We empirically verify our findings on the MNIST task.

 We believe that the approach of studying challenging theoretical problems in deep learning through simplified learning tasks can be fruitful. For future work, it would be interesting to consider more complex tasks, e.g., filters of higher dimension or non-binary data, to better understand overparameterization.

\subsubsection*{Acknowledgments}
This research is supported by the Blavatnik Computer Science Research Fund  and by the Yandex Initiative in Machine Learning.

\bibliography{xor}
\bibliographystyle{icml2017}

\appendix


\section{Experiment in Figure \ref{fig:xor_over_param}}
\label{sec:fig1_exp}

We tested the generalization performance in the setup of Section\ref{sec:problem_formulation}. We considered networks with number of channels 4,6,8,20,50,100 and 200. The distribution in this setting has $p_+ = 0.5$ and $p_- = 0.9$ and the training sets are of size 12 (6 positive, 6 negative). Note that in this case the training set contains non-diverse points with high probability. The ground truth network can be realized by a network with 4 channels. For each number of channels we trained a convolutional network 100 times and averaged the results.  In each run we sampled a new training set and new initialization of the weights according to a gaussian distribution with mean 0 and standard deviation 0.00001. For each number of channels $c$, we ran gradient descent with learning rate $\frac{0.04}{c}$ and stopped it if it did not improve the cost for 20 consecutive iterations or if it reached 30000 iterations. The last iteration was taken for the calculations. We plot both average test error over all 100 runs and average test error only over the runs that ended at 0\% train error. In this case, for each number of channels 4,6,8,20,50,100,200 the number of runs in which gradient descent converged to a $0\%$ train error solution is 62, 79, 94, 100, 100, 100, 100, respectively.

\section{Proofs for Section \ref{sec:xor}}
\label{sec:xor_appendix}

In the XOR problem, we are given a training set $S = \left\{(\xx_i,y_i)\right\}_{i=1}^4 \subseteq \{\pm1\}^2 \times \{\pm1\}^2$ consisting of points  $\xx_1 = (1,1)$, $\xx_2 = (-1,1)$, $\xx_3 = (-1,-1)$, $\xx_4 = (1,-1)$ with labels $y_1 = 1$, $y_2 = -1$, $y_3 = 1$ and $y_4 = -1$, respectively. Our goal is to learn the XOR function $f^*:\{\pm 1\}^2 \rightarrow \{\pm 1\}$, such that $f^*(\xx_i) = y_i$ for $1 \le i \le 4$, with a neural network and gradient descent.

\paragraph{\underline{Neural Architecture:}} For this task we consider the following two-layer fully connected network.

\begin{equation}
	\label{eq:xor_network}
	\net(\xx)=\sum_{i=1}^{k}\left[\sigma\left(\wvec{i}\cdot \xx \right) - \sigma\left(\uvec{i}\cdot \xx \right)\right]
\end{equation}
where $W \in \reals^{2k \times 2}$ is the weight matrix whose rows are the $\wvec{i}$ vectors followed by the $\uvec{i}$ vectors, and $\sigma(x)=\max\{0,x\}$ is the ReLU activation applied element-wise. We note that $f^*$ can be implemented with this network for $k=2$ and this is the minimal $k$ for which this is possible. Thus we refer to $k > 2$ as the overparameterized case.

\paragraph{\underline{Training Algorithm:}} The parameters of the network $\net(\xx)$ are learned using gradient descent on the hinge loss objective.
We use a constant learning rate $\eta \le \frac{\ceta}{k}$, where $\ceta < \frac{1}{2}$. The parameters $\net$ are initialized as IID Gaussians with zero mean and standard deviation $\sigma_g \le \frac{\ceta}{16k^{3/2}}$. We consider the hinge-loss objective:
$$\ell(W)=\sum_{(\xx,y)\in S}{\max\{1-y\net(\xx),0\}}$$
where optimization is only over the first layer of the network. We note that for $k \ge 2$ any global minimum $W$ of $\ell$ satisfies $\ell(W) = 0$ and $\sign(\net(\xx_i)) = f^*(\xx_i)$ for $1 \le i \le 4$.

\paragraph{\underline{Notations:}} We will need the following notations. Let $W_t$ be the weight matrix at iteration $t$ of gradient descent. For $1 \leq i \leq k$, denote by $\wvec{i}_t \in \reals^2$ the $i^{th}$ weight vector at iteration $t$. Similarly  we define $\uvec{i}_t \in \reals^2$ to be the $k + i$ weight vector at iteration $t$. 
For each point $\xx_i \in S$ define the following sets of neurons:
$$W_t^+(i) = \left\{j \mid \wvec{j}_t\cdot \xx_i > 0\right\}$$
$$W_t^-(i) = \left\{j \mid \wvec{j}_t\cdot \xx_i < 0\right\}$$
$$U_t^+(i) = \left\{j \mid \uvec{j}_t\cdot \xx_i > 0\right\}$$
$$U_t^-(i) = \left\{j \mid \uvec{j}_t\cdot \xx_i < 0\right\}$$

and for each iteration $t$, let $a_i(t)$ be the number of iterations $0 \le t' \le t$ such that $y_i\nett{t'}(\xx_i) < 1$.

\subsection{Overparameterized Network}

\begin{lem}
	\label{lem:init_num}
	\textbf{Exploration at initialization.} With probability at least $1-8e^{-8}$, for all $1 \le j \le 4$ $$\frac{k}{2} -2\sqrt{k}\le \abs{W_0^+(j)},\abs{U_0^+(j)} \le \frac{k}{2} +2\sqrt{k}$$
\end{lem}
\begin{proof}
	Without loss of generality consider $\abs{W_0^+(1)}$. Since the sign of a one dimensional Gaussian random variable is a Bernoulli random variable, we get by Hoeffding's inequality $$\prob\left(\left|\abs{W_0^+(1)}-\frac{k}{2}\right| < 2\sqrt{k}\right)\le 2 e^{-\frac{2(2^2 k)}{k}} = 2e^{-8}$$
	Since $\abs{W_0^+(1)} + \abs{W_0^+(3)} = k$ with probability 1, we get that if $\left|\abs{W_0^+(1)}-\frac{k}{2}\right| < 2\sqrt{k}$ then $\left|\abs{W_0^+(3)}-\frac{k}{2}\right| < 2\sqrt{k}$.
	The result now follows by symmetry and the union bound.
\end{proof}

\begin{lem}
	\label{lem:init_bound_disj_xor}
	With probability $\ge 1-\frac{\sqrt{2k}}{\sqrt{\pi} e^{8k}}$, for all $1 \le j \le k $ and $1 \le i \le 4$ it holds that  $\abs{\wvec{j}_0 \cdot \xx_i} \le \frac{\sqrt{2}\eta}{4}$ and $\abs{\uvec{j}_0 \cdot \xx_i} \le \frac{\sqrt{2}\eta}{4}$.
\end{lem}
\begin{proof}
	Let $Z$ be a random variable distributed as $\mathcal{N}(0,\sigma^2)$. Then by Proposition 2.1.2 in \cite{vershynin2017high}, we have  $$\probarg{\abs{Z} \ge t} \le \frac{2\sigma}{\sqrt{2\pi}t}e^{-\frac{t^2}{2\sigma^2}} $$
	Therefore, for all $1 \le j \le k $ and $1 \le i \le 4$, $$\probarg{\abs{\wvec{j}_0 \cdot \xx_i} \ge \frac{\sqrt{2}\eta}{4}} \le \frac{1}{\sqrt{32\pi k}}e^{-8k}$$
	and 
	$$\probarg{\abs{\uvec{j}_0 \cdot \xx_i} \ge \frac{\sqrt{2}\eta}{4}} \le \frac{1}{\sqrt{32\pi k}}e^{-8k}$$
	The result follows by applying a union bound over all $2k$ weight vectors and the four points $\xx_i$, $1 \le i \le 4$. 
\end{proof}

\begin{lem}
	\label{lem:clustering_xor_restated}
	\textbf{Clustering Dynamics. Lemma \ref{lem:clustering_xor} restated and extended.} With probability $\ge 1-\frac{\sqrt{2k}}{\sqrt{\pi} e^{8k}}$, for all $t \ge 0$ there exists $\alpha_i$, $1 \le i \le 4$ such that $\left|\alpha_i\right| \le \eta$ and the following holds:
	\begin{enumerate}
		\item For $i \in \{1,3\}$ and $j \in W_0^+(i)$, it holds that $\wvec{j}_t =  \wvec{j}_0 + a_i(t)\eta\xx_i + \alpha_i\xx_2$.
		\item For $i \in \{2,4\}$ and $j \in U_0^+(i)$, it holds that $\uvec{j}_t =  \uvec{j}_0 + a_i(t)\eta\xx_i + \alpha_i\xx_1$.
	\end{enumerate}
\end{lem}
\begin{proof}
	By Lemma \ref{lem:init_bound_disj_xor}, with probability $\ge 1-\frac{\sqrt{2k}}{\sqrt{\pi} e^{8k}}$, for all $1 \le j \le k $ and $1 \le i \le 4$ it holds that  $\abs{\wvec{j}_0 \cdot \xx_i} \le \frac{\sqrt{2}\eta}{4}$ and $\abs{\uvec{j}_0 \cdot \xx_i} \le \frac{\sqrt{2}\eta}{4}$. It suffices to prove the claim for $W_t^+(1)$. The other cases follow by a symmetry. The proof is by induction. Assume that $j \in W_t^+(1)$. For $t=0$ the claim holds with $\alpha_1^t = 0$. For a point $(\xx,y)$ let $\ell_{(\xx,y)} = \max\{1-y\net(\xx),0\}$. Then it holds that $\frac{\partial \ell_{(\xx,y)}}{\partial \wvec{i}}\left(W\right) = -y\sigma'(\wvec{i}\cdot \xx)\xx\mathbbm{1}_{y\net(\xx) < 1}$.
	Assume without loss of generality that $\alpha_1^t > 0$. Define $\beta_1 = \mathbbm{1}_{\net(\xx_1) < 1}$ and $\beta_2 = \mathbbm{1}_{\net(\xx_2) > -1}$. Using these notations, we have 
	\begin{align*}
		\wvec{j}_{t+1} &=  \wvec{j}_t + \beta_1\eta\xx_1 - \beta_2\eta\xx_2 \\ &= \wvec{j}_0 + (a_i(t)+\beta_1)\xx_i + (\alpha_i-\beta_2\eta)\xx_2
	\end{align*}
	and for any values of $\beta_1, \beta_2 \in \{0,1\}$ the induction step follows.

\end{proof}

For each point $\xx_i$, define the following sums:
$$S_t^+(i) = \sum_{j\in W_t^+(i)}\sigma\left(\wvec{j}_t\cdot \xx_i \right)$$
$$S_t^-(i) = \sum_{j\in W_t^-(i)}\sigma\left(\wvec{j}_t\cdot \xx_i \right)$$
$$R_t^+(i) = \sum_{j\in U_t^+(i)}\sigma\left(\uvec{j}_t\cdot \xx_i \right)$$
$$R_t^-(i) = \sum_{j\in U_t^-(i)}\sigma\left(\uvec{j}_t\cdot \xx_i \right)$$

We will prove the following lemma regarding $S_t^+(1), S_t^-(1), R_t^+(1), R_t^-(1)$ for $i = 1$. By symmetry, analogous lemmas follow for $i \neq 1$.

\begin{lem}
	\label{lem:init_summations}
	The following holds with probability $\ge 1-\frac{\sqrt{2k}}{\sqrt{\pi} e^{8k}}$:
	\begin{enumerate}
		\item For all $t \ge 0$, $R_t^+(1) + R_t^-(1) \le k \eta$.
		\item Let $t \ge 0$ then $S_t^-(1)	= 0$. Furthermore, if $-y\nett{t}(\xx_1) < 1$, then $S_{t+1}^+(1) \ge S_t^+(1) + \abs{W_0^+(1)}\eta$. Otherwise, if $-y\nett{t}(\xx_1) \ge 1$ then $S_{t+1}^+(1) = S_t^+(1)$.
	\end{enumerate}
\end{lem}
\begin{proof}
	\begin{enumerate}
		\item Assume by contradiction that there exists $t > 0$, such that $R_t^+(1) + R_t^-(1) > k\eta$. It follows that, without loss of generality, there exists $j \in  U_t^+(1)$ such that $\sigma\left(\uvec{j}_t\cdot \xx_1 \right) > \eta$. However, this contradicts Lemma \ref{lem:clustering_xor_restated}.
		\item All of the claims are direct consequences of Lemma \ref{lem:clustering_xor_restated}.
	\end{enumerate}	
\end{proof}

\begin{prop}
	\label{prop:convergence}
	Assume that $k > 16$. With probability $\ge 1 - \frac{\sqrt{2k}}{\sqrt{\pi} e^{8k}} - 8e^{-8}$, for all $i$, if until iteration $T$ there were at least $l \ge \frac{4\sqrt{k}}{\sqrt{k}-2}$ iterations, in which $-y\nett{t}(\xx_i) < 1$, then it holds that $-y\nett{t}(\xx_i) \ge 1$ for all $t \ge T$.
\end{prop}
\begin{proof}
	Without loss of generality assume that $i = 1$. By Lemma \ref{lem:init_summations} and Lemma \ref{lem:init_bound_disj}, with probability $\ge 1 - \frac{\sqrt{2k}}{\sqrt{\pi} e^{8k}} - 8e^{-8}$, if $-y\nett{t}(\xx_1) < 1$ then $S_{t+1}^+(1) \ge S_t^+(1) + \left(\frac{k}{2}-2\sqrt{k}\right)\eta$. Therefore, by Lemma \ref{lem:init_summations}, for all $t \ge T$
	\begin{align*}
		\nett{t}(\xx_1) &= S_{t}^+(1) + S_{t}^-(1) - R_{t}^+(1) - R_{t}^-(1) \\ &\ge \left(\frac{k}{2}-2\sqrt{k}\right)l\eta - k\eta \\ &\ge 1
	\end{align*}
	where the last ineqaulity follows by the assumption on $l$.
\end{proof}

\begin{thm} 
	\label{thm:xor_overparam_restated}
	\textbf{Convergence and clustering. Theorem \ref{thm:xor_overparam} restated. }Assume that $k > 16$. With probability $\ge 1 - \frac{\sqrt{2k}}{\sqrt{\pi} e^{8k}} - 8e^{-8}$, after at most $T \le \frac{16\sqrt{k}}{\sqrt{k}-2}$ iterations, gradient descent converges to a global minimum. Furthermore, for $i \in \{1,3\}$ and all $j \in W_0^+(i)$, the angle between  $\wvec{j}_T$ and $\xx_i$ is at most $\arccos\left(\frac{1-2\ceta}{1+\ceta}\right)$. Similarly, for $i \in \{2,4\}$ and all $j \in U_0^+(i)$, the angle between  $\uvec{j}_T$ and $\xx_i$ is at most $\arccos\left(\frac{1-2\ceta}{1+\ceta}\right)$.
\end{thm}
\begin{proof}
	Proposition \ref{prop:convergence} implies that there are at most $\frac{16\sqrt{k}}{\sqrt{k}-2}$ iterations in which there exists $(\xx_i,y_i)$ such that $y_i\nett{t}(\xx_i) < 1$. After at most that many iterations, gradient descent converges to a global minimum.
	
	Without loss of generality, we prove the clustering claim for $i = 1$ and all $j \in W_0^+(1)$. At a global minimum, $\nett{T}(\xx_1) \ge 1$. Therefore, by Lemma \ref{lem:clustering_xor_restated} and Lemma \ref{lem:init_summations} it follows that $$2\eta(a_i(T) + 1) \abs{W^+_0(1)} \ge S_t^+(1) \ge 1$$
	and thus $a_i(T) \ge \frac{1}{2\ceta} - 1$. Therefore, for any $j \in W^+_0(1) $, the cosine of the angle between $\wvec{j}_T$ and $\xx_1$ is at least
	$$\frac{ (\wvec{j}_0 + a_1(T)\eta\xx_1 + \alpha_1^t\xx_2) \cdot \xx_1}{\sqrt{2}(\norm{\wvec{j}_0} + \sqrt{2}a_i(T)\eta + \sqrt{2}\eta)} \ge \frac{2a_1(T)}{2a_1(T) + 3} \ge \frac{1-2\ceta}{1+\ceta}$$
	
	where we used the triangle inequality and Lemma \ref{lem:clustering_xor_restated}. The claim follows.
\end{proof}

\subsection{Small Network}

\begin{lem}
	\label{lem:non_exploration_init}
	\textbf{Non-exploration at initialization.} With probability at least $0.75$, there exists $i \in \{1,3\}$ such that  $W_0^+(i) = \emptyset$ or $i \in \{2,4\}$ such that $U_0^+(i) = \emptyset$.
\end{lem}
\begin{proof}
	Since the sign of a one dimensional Gaussian random variable is a Bernoulli random variable, the probability that $W_0^+(i) \neq \emptyset$ for $i \in \{1,3\}$ and $U_0^+(i) \neq \emptyset$ for $i \in \{2,4\}$ is $\frac{1}{4}$. The claim follows.
\end{proof}

\begin{thm}
	Assume that $k = 2$. With probability $\ge 0.75$, gradient descent converges to a local minimum.
\end{thm}
\begin{proof}
	As in the proof of Theorem \ref{thm:xor_overparam}, for $i \in \{1,3\}$ if $W_0^+(i) \neq \emptyset$, then eventually, $y_i\nett{t}(\xx_i) \ge 1$. Similarly, for $i \in \{2,4\}$ if $U_0^+(i) \neq \emptyset$, then eventually, $y_i\nett{t}(\xx_i) \ge 1$. However, if without loss of generality $W_0^+(1) = \emptyset$, then for all $t$,
	\begin{align*}
		\nett{t}(\xx_1) = S_{t}^+(1) + S_{t}^-(1) - R_{t}^+(1) - R_{t}^-(1) \le 0
	\end{align*}
	Furthermore, there exists the first iteration $t'$ such that $y_i\nett{t'}(\xx_i) \ge 1$ for $i=3$ (since $W_0^+(3) \neq \emptyset$) and any $i \in \{2,4\}$ such that $U_0^+(i) \neq \emptyset$. Then, in iteration $t'+1$ for all $1 \le j \le 2$ it holds that $\uvec{j}_{t'+1}\xx_i < 0$ and $\wvec{j}_{t'+1}\xx_i < 0$ for $i=1$ or $i \in \{2,4\}$ such that $U_0^+(i) = \emptyset$. Therefore at $t'+1$ we are at a local minimum.
\end{proof}

\section{Proofs and Experiments for Section \ref{sec:xord_problem_formulation}}
\label{sec:vc}

\subsection{VC Dimension}
As noted in Remark \ref{rem:expr}, the VC dimension of the model we consider is at most $15$. 
To see this, we first define for any $\zz \in \{\pm 1\}^{2d}$ the set $P_{\zz} \subseteq \{\pm1\}^2$ which contains all the distinct two dimensional binary patterns that $\zz$ has. For example, for a positive diverse point $\zz$ it holds that $P_{\zz} = \{\pm1\}^2$. Now, for any points $\zz^{(1)}, \zz^{(2)} \in \{\pm 1\}^{2d}$ such that $P_{\zz^{(1)}}=P_{\zz^{(2)}}$ and for any filter $\ww \in \reals^{2}$ it holds that $\max_{j}\sigma\left(\ww\cdot \zz^{(1)}_j \right) = \max_{j}\sigma\left(\ww\cdot \zz^{(2)}_j \right)$.
Therefore, for any $W$, $\net(\zz^{(1)}) = \net(\zz^{(2)})$. Specifically, this implies that if both $\zz^{(1)}$ and $\zz^{(2)}$ are diverse then $\net(\zz^{(1)}) = \net(\zz^{(2)})$. Since there are 15 non-empty subsets of $\{\pm 1\}^2$, it follows that for any $k$ the network can shatter a set of at most 15 points, or equivalently, its VC dimension is at most 15. Despite these expressive power limitations,  there is a generalization gap between small and large networks in this setting, as can be seen in Figure \ref{fig:xor_over_param}.

\subsection{Hinge Loss Confidence}
\label{sec:hinge}

Figure \ref{fig:hinge} shows that setting $\gamma = 5$ gives better performance than setting $\gamma = 1$ in the XORD problem. The setting is similar to the setting of Section \ref{sec:fig1_exp}. Each point is an average test error of 100 runs.

\begin{figure}	
	
	\centering
	\includegraphics[width=0.4\linewidth]{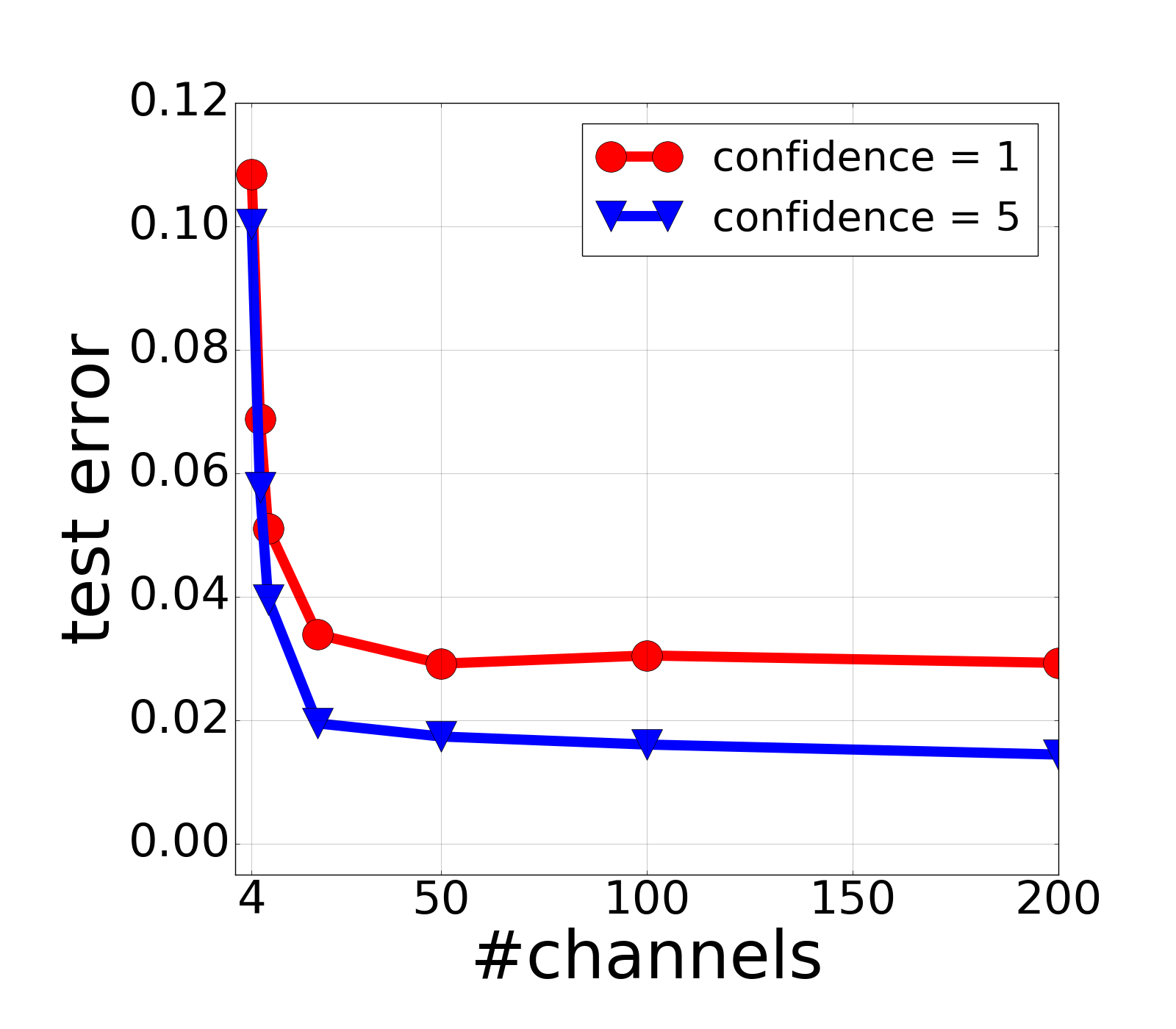}
	\caption{\small{Higher confidence of hinge-loss results in better performance in the XORD problem.}}
	\label{fig:hinge}
\end{figure}

\section{Experiments for Section \ref{sec:xord_experiments}}

Here we show an example of a training set that contains a non-diverse negative point. In total, the training set has 6 positive points and 6 negative points. We implemented the setting of Section \ref{sec:xord_problem_formulation} and ran gradient descent on this training set. In Figure \ref{fig:exp_xordnd} we show the results. The large network recovers $f^*$, while the small does not. This is despite the fact that both networks achieve zero training error.

\begin{figure*}[th]
	\begin{subfigure}{.24\textwidth}
		\centering
		\includegraphics[width=1.0\linewidth]{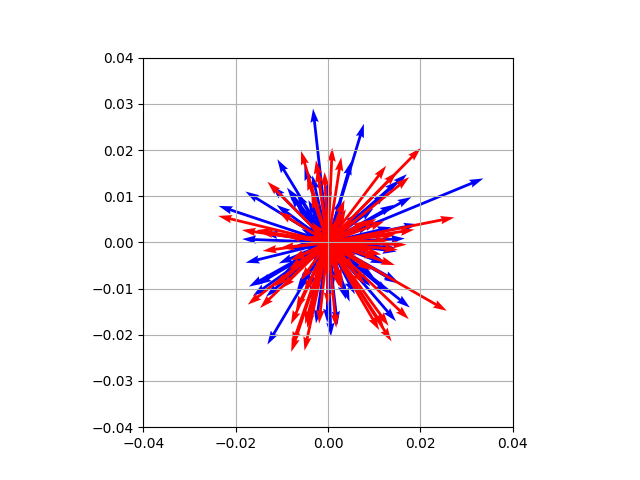}
		\caption{}
		\label{fig:exp_xordnd1}
	\end{subfigure}%
	\begin{subfigure}{.24\textwidth}
		\centering
		\includegraphics[width=1.0\linewidth]{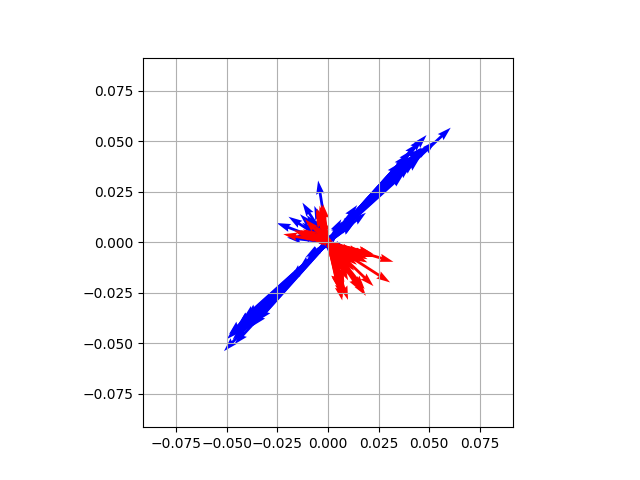}
		\caption{}
		\label{fig:exp_xordnd2}
	\end{subfigure}%
	\begin{subfigure}{.24\textwidth}
		\centering
		\includegraphics[width=1.0\linewidth]{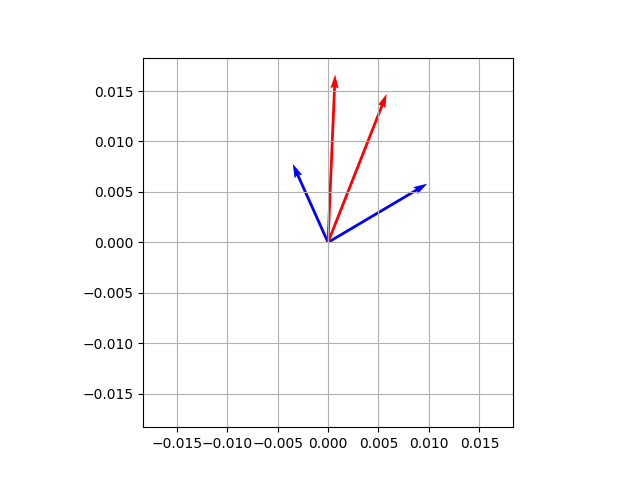}
		\caption{}
		\label{fig:exp_xordnd3}
	\end{subfigure}
	\begin{subfigure}{.24\textwidth}
		\centering
		\includegraphics[width=1.0\linewidth]{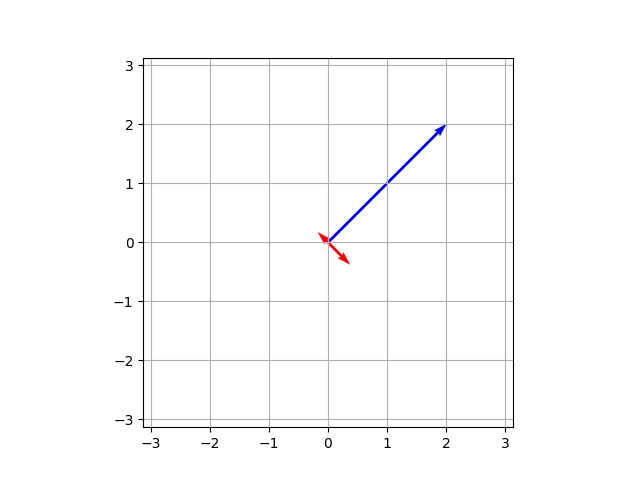}
		\caption{}
		\label{fig:exp_xordnd4}
	\end{subfigure}
	\caption{\small{Overparameterization and generalization in XORD problem. The vectors in blue are the vectors $\wvec{i}_t$ and in red are the vectors $\uvec{i}_t$. (a) Exploration at initialization (t=0) for $k=100$ (b) Clustering and convergence to global minimum that recovers $f^*$ for $k=100$ (c) Non-sufficient exploration at initialization (t=0) for $k=2$. (d) Convergence to global minimum with non-zero test error for $k=2$.}}
	\label{fig:exp_xordnd}
\end{figure*}

\section{Proof of Theorem \ref{thm:main}}
\label{sec:proof_overparam}

We first restate the theorem.

\begin{thm}
	\label{thm:main_restated}
	(\textbf{Theorem \ref{thm:main} restated and extended.) }With probability at least $\left(1-c - 16e^{-8}\right)$ after running gradient descent for $T \ge \frac{28(\gamma + 1 + 8\ceta)}{\ceta}$ iterations, it converges to a global minimum which satisfies $\sign\left(\nett{T}(\xx)\right) = f^*(\xx)$ for all $\xx \in \{\pm1\}^{2d}$. Furthermore, for $i \in \{1,3\}$ and all $j \in W_0^+(i)$, the angle between  $\wvec{j}_T$ and $\pp_i$ is at most $\arccos\left(\frac{\gamma - 1-2\ceta}{\gamma-1+\ceta}\right)$.
\end{thm}

We will first need a few notations. Define $\pp_1=(1,1),  \xx_2=(1,-1), \pp_3=(-1,-1), \pp_4=(-1,1)$ and the following sets:
$$W_t^+(i) = \left\{j \mid \argmax_{1 \le l \le 4} \wvec{j}_t\cdot \pp_l = i \right\},\,\,U_t^+(i) = \left\{j \mid \argmax_{1 \le l \le 4} \uvec{j}_t\cdot \pp_l = i \right\}$$
$$W_t^-(i) = \left\{j \mid \argmax_{l \in \{2,4\}} \wvec{j}_t\cdot \pp_l = i \right\},\,\,U_t^-(i) = \left\{j \mid \argmax_{l \in \{2,4\}} \uvec{j}_t\cdot \pp_l = i \right\}$$

We can use these definitions to express more easily the gradient updates. Concretely, let $j \in W_t^+(i_1)\cap W_t^-(i_2)$ then the gradient update is given as follows:\footnote{Note that with probability 1, $\sigma'(\wvec{j}_t\cdot\pp_{i_1})=1$, $\sigma'(\wvec{j}_t\cdot\pp_{i_2})=1$ for all $t$, and therefore we omit these from the gradient update. This follows since $\sigma'(\wvec{j}_t\cdot\pp_{i_1})=0$ for some $t$ if and only if $\wvec{j}_0\cdot\pp_{i_1}$ is an integer multiple of $\eta$.}
\begin{equation}
	\label{eq:w_gradient}
	\wvec{j}_{t+1} = \wvec{j}_{t} + \eta \pp_{i_1}\mathbbm{1}_{\net(\xx^+) < \gamma} - \eta \pp_{i_2}\mathbbm{1}_{\net(\xx^-) < 1}
\end{equation}
Similarly, for $j \in U_t^+(i_1)\cap U_t^-(i_2)$ the gradient update is given by:
\begin{equation}
	\label{eq:u_gradient}
	\uvec{j}_{t+1} = \uvec{j}_{t} - \eta \pp_{i_1}\mathbbm{1}_{\net(\xx^+) < \gamma} + \eta \pp_{i_2}\mathbbm{1}_{\net(\xx^-) < 1}
\end{equation}

We denote by $\xx^+$ a positive diverse point and $\xx^-$ a negative diverse point. Define the following sums for $\phi\in\{+,-\}$:
$$S_t^\phi = \sum_{j \in W_t^+(1) \cup W_t^+(3)}\left[\max\left\{\sigma\left(\wvec{j}\cdot \xx_1^\phi \right),...,\sigma\left(\wvec{j}\cdot \xx_d^\phi \right)\right\}\right]$$
$$P_t^\phi = \sum_{j \in U_t^+(1) \cup U_t^+(3)}\left[\max\left\{\sigma\left(\uvec{j}\cdot \xx_1^\phi \right),...,\sigma\left(\uvec{j}\cdot \xx_d^\phi \right)\right\}\right]$$

\begin{align*}
	R_t^\phi &= \sum_{j \in W_t^+(2) \cup W_t^+(4)}\left[\max\left\{\sigma\left(\wvec{j}\cdot \xx_1^\phi \right),...,\sigma\left(\wvec{j}\cdot \xx_d^\phi \right)\right\}\right] \\ &- \sum_{j \in U_t^+(2) \cup U_t^+(4)}\left[\max\left\{\sigma\left(\uvec{i}\cdot \xx_1^\phi \right),...,\sigma\left(\uvec{i}\cdot \xx_d^\phi \right)\right\}\right]
\end{align*}

Note that $R_t^+ = R_t^-$ since for $\zz \in \{\xx^+, \xx^-\}$ there exists $i_1,i_2$ such that $\zz_{i_1} = \pp_2$, $\zz_{i_2} = \pp_4$.

Without loss of generality, we can assume that the training set consists of one positive diverse point $\xx^+$ and one negative diverse point $\xx^-$. This follows since the network and its gradient have the same value for two different positive diverse points and two different negative points. Therefore, this holds for the loss function defined in \eqref{eq:loss_function} as well.

We let $a^+(t)$ be the number of iterations $0 \le t' \le t$ such that $\nett{t'}(\xx^+) < \gamma$. 

We will now proceed to prove the theorem. In Section \ref{sec:init_guarantees} we prove results on the filters at initialization. In Section \ref{sec:auxiliary} we prove several lemmas that exhibit the clustering dynamics. In Section \ref{sec:bounds1} we prove upper bounds on $S_t^-$, $P_t^+$ and $P_t^-$ for all iterations $t$. In Section \ref{sec:dynamics} we characterize the dynamics of $S_t^+$ and in Section \ref{sec:bounds2} we prove an upper bound on it together with upper bounds on $\nett{t}(\xx^+)$ and $-\nett{t}(\xx^-)$ for all iterations $t$.

We provide an optimization guarantee for gradient descent in Section \ref{sec:opt}. We prove generalization guarantees for the points in the positive class and negative class in Section \ref{sec:positive_class} and Section \ref{sec:negative_class}, respectively. We complete the proof of the theorem in Section \ref{sec:finishing_proof} with proofs for the clustering effect at the global minimum.

\subsubsection{Initialization Guarantees}
\label{sec:init_guarantees}

\begin{lem}
	\label{lem:init_num_disj}
	\textbf{Exploration. Lemma \ref{lem:init_num_disj_xord} restated and extended.} With probability at least $1-4e^{-8}$, it holds that $$\abs{\abs{W_0^+(1) \cup W_0^+(3)} - \frac{k}{2}}\le 2\sqrt{k}$$ and $$\abs{\abs{U_0^+(1) \cup U_0^+(3)} - \frac{k}{2}}\le 2\sqrt{k}$$
\end{lem}
\begin{proof}
	Without loss of generality consider $\abs{W_0^+(1) \cup W_0^+(3)} $. Since $\probarg{j \in W_0^+(1) \cup W_0^+(3)} = \frac{1}{2}$, we get by Hoeffding's inequality $$\probarg{\left|\abs{W_0^+(1) \cup W_0^+(3)} - \frac{k}{2}\right| < 2\sqrt{k}}\le 2 e^{-\frac{2(2^2 k)}{k}} = 2e^{-8}$$
	The result now follows by the union bound.
\end{proof}

\begin{lem}
	\label{lem:init_bound_disj}
	With probability $\ge 1-\frac{\sqrt{2k}}{\sqrt{\pi} e^{8k}}$, for all $1 \le j \le k $ and $1 \le i \le 4$ it holds that  $\abs{\wvec{j}_0 \cdot \pp_i} \le \frac{\sqrt{2}\eta}{4}$ and $\abs{\uvec{j}_0 \cdot \pp_i} \le \frac{\sqrt{2}\eta}{4}$.
\end{lem}
\begin{proof}
	Let $Z$ be a random variable distributed as $\mathcal{N}(0,\sigma^2)$. Then by Proposition 2.1.2 in \cite{vershynin2017high}, we have  $$\probarg{\abs{Z} \ge t} \le \frac{2\sigma}{\sqrt{2\pi}t}e^{-\frac{t^2}{2\sigma^2}} $$
	Therefore, for all $1 \le j \le k $ and $1 \le i \le 4$, $$\probarg{\abs{\wvec{j}_0 \cdot \pp_i} \ge \frac{\sqrt{2}\eta}{4}} \le \frac{1}{\sqrt{32\pi k}}e^{-8k}$$
	and 
	$$\probarg{\abs{\uvec{j}_0 \cdot \pp_i} \ge \frac{\sqrt{2}\eta}{4}} \le \frac{1}{\sqrt{32\pi k}}e^{-8k}$$
	The result follows by applying a union bound over all $2k$ weight vectors and the four points $\pp_i$, $1 \le i \le 4$. 
\end{proof}

From now on we assume that the highly probable event in Lemma \ref{lem:init_bound_disj} holds.

\begin{lem}
	\label{lem:iteration_2}
	$\nett{t}(\xx^+) < 1$ and $-\nett{t}(\xx^-) < 1$ for $0 \le t \le 2$.
\end{lem}
\begin{proof}
	By Lemma \ref{lem:init_bound_disj} we have 
	\begin{align*}
		\nett{0}(\xx^+) &= \sum_{i=1}^{k}\left[\max\left\{\sigma\left(\wvec{i}_0\cdot \xx_1^+ \right),...,\sigma\left(\wvec{i}_0\cdot \xx_d^+ \right)\right\} - \max\left\{\sigma\left(\uvec{i}_0\cdot \xx_1^+ \right),...,\sigma\left(\uvec{i}_0\cdot \xx_d^+ \right)\right\}\right] \\ &\le \frac{\eta k}{4} < \gamma
	\end{align*}
	and similarly $-\nett{0}(\xx^-) < 1$. Therefore, by \eqref{eq:w_gradient} and \eqref{eq:u_gradient} we get:
	\begin{enumerate}
		\item For $i \in \{1,3\}$, $l \in \{2,4\}$, $j \in W_0^+(i) \cap W_0^-(l)$, it holds that $\wvec{j}_{1} = \wvec{j}_0 -\eta \pp_l + \eta \pp_i$.
		\item For $i \in \{2,4\} $ and $j \in W_0^+(i)$, it holds that $\wvec{j}_{1} = \wvec{j}_0$.
		\item For $i \in \{1,3\}$, $l \in \{2,4\}$, $j \in U_0^+(i) \cap U_0^-(l)$, it holds that $\uvec{j}_{1} = \uvec{j}_0  -\eta\pp_i + \eta \pp_l $.
		\item For $i \in \{2,4\} $ and $j \in U_0^+(i)$, it holds that $\uvec{j}_{2} = \uvec{j}_0$.
	\end{enumerate}
	Applying Lemma \ref{lem:init_bound_disj} again and using the fact that $\eta \le \frac{1}{8k}$ we have $\nett{1}(\xx^+) < \gamma$ and $-\nett{1}(\xx^-) < 1$. Therefore we get,
	\begin{enumerate}
		\item For $i \in \{1,3\}$, $l \in \{2,4\}$, $j \in W_0^+(i) \cap W_0^-(l)$, it holds that $\wvec{j}_{2} = \wvec{j}_0 + 2\eta \pp_i$.
		\item For $i \in \{2,4\} $ and $j \in W_0^+(i)$, it holds that $\wvec{j}_{2} = \wvec{j}_0$.
		\item For $i \in \{1,3\}$, $l \in \{2,4\}$, $j \in U_0^+(i) \cap U_0^-(l)$, it holds that $\uvec{j}_{2} = \uvec{j}_0  -\eta\pp_i + \eta \pp_l$.
		\item For $i \in \{2,4\} $ and $j \in U_0^+(i)$, it holds that $\uvec{j}_{2} = \uvec{j}_0$.
	\end{enumerate}
	As before, by Lemma \ref{lem:init_bound_disj} we have $\nett{2}(\xx^+) < \gamma$ and $-\nett{2}(\xx^-) < 1$.
\end{proof}

\subsubsection{Clustering Dynamics Lemmas}
\label{sec:auxiliary}

In the following lemmas we assume that the highly probable event in Lemma \ref{lem:init_bound_disj} holds. We therefore do not mention the probability in the statements of the lemmas.

\begin{lem}
	\label{lem:same_w}
	\textbf{Clusetering. Lemma \ref{lem:clustering_xord} restated and extended.} For all $t \ge 0$ there exists $\alpha_i^t$, $i \in \{1,3\}$ such that $\left|\alpha_i^t\right| \le \eta$ and the following holds:
	\begin{enumerate}
		\item For $i \in \{1,3\}$ and $j \in W_0^+(i)$, it holds that $\wvec{j}_t =  \wvec{j}_0 + a^+(t)\eta\pp_i + \alpha_i^t\pp_2$.
		\item For $i \in \{2,4\}$ and $j \in W_0^+(i)$, it holds that $\wvec{j}_{t} = \wvec{j}_{0} + m \pp_2$ for $m \in \mathbb{Z}$.
		\item $W_t^+(i) = W_0^+(i)$ for $i \in \{1,3\}$.
	\end{enumerate}
\end{lem}
\begin{proof}
	By Lemma \ref{lem:init_bound_disj}, with probability $\ge 1-\frac{\sqrt{2k}}{\sqrt{\pi} e^{8k}}$, for all $1 \le j \le k $ and $1 \le i \le 4$ it holds that  $\abs{\wvec{j}_0 \cdot \xx_i} \le \frac{\sqrt{2}\eta}{4}$ and $\abs{\uvec{j}_0 \cdot \xx_i} \le \frac{\sqrt{2}\eta}{4}$. We will first prove the first claim and that $W_0^+(i) \subseteq W_t^+(i)$ for all $t \ge 1$. To prove this, we will show by induction on $t \ge 1$, that for all $j \in W_0^+(i) \cap W_0^+(l)$, where $l \in \{2,4\}$ the following holds:
	\begin{enumerate}
		\item $j \in W_t^+(i)$.
		\item $\wvec{j}_{t} \cdot \pp_l = \wvec{j}_{0} \cdot \pp_l - \eta$ or $\wvec{j}_{t} \cdot \pp_l = \wvec{0}_{t} \cdot \pp_l$.
		\item  $\wvec{j}_t =  \wvec{j}_0 + a^+(t)\eta\pp_i + \alpha_i\pp_2$
		\item $\wvec{j}_{t} \cdot \pp_i > \eta$.
	\end{enumerate}
	
	The claim holds for $t=1$ by the proof of Lemma \ref{lem:iteration_2}. Assume it holds for $t=T$. By the induction hypothesis there exists an $l' \in \{2,4\}$ such that $j \in W_{T}^+(i)\cap W_{T}^-(l')$. By \eqref{eq:w_gradient} we have,
	\begin{equation}
		\label{eq:gradient_13}
		\wvec{j}_{T+1} = \wvec{j}_{T} + a\eta\pp_i + b\eta\pp_{l'}
	\end{equation}
	where $a = a^+(t+1)-a^+(t)$ and $b \in \{-1,0\}$. From this follows the third claim of the induction proof and the first claim of the lemma. 
	
	If $\wvec{j}_{T} \cdot \pp_l = \wvec{j}_{0} \cdot \pp_l$ then $l' = l$ and either $\wvec{j}_{T+1} \cdot \pp_l = \wvec{j}_{0} \cdot \pp_l$ if $b=0$ or $\wvec{j}_{T+1} \cdot \pp_l = \wvec{j}_{0} \cdot \pp_l - \eta$ if $b=-1$. Otherwise, assume that $\wvec{j}_{T} \cdot \pp_l = \wvec{j}_{0} \cdot \pp_l - \eta$. By Lemma \ref{lem:init_bound_disj} we have $0 < \wvec{j}_{0} \cdot \pp_l < \frac{\sqrt{2}\eta}{4}$.  Therefore $-\eta < \wvec{j}_{T} \cdot \pp_l < 0$ and $l' \neq l$. It follows that either $\wvec{j}_{T+1} \cdot \pp_l = \wvec{j}_{0} \cdot \pp_l - \eta$ if $b=0$ or $\wvec{j}_{T+1} \cdot \pp_l = \wvec{j}_{0} \cdot \pp_l$ if $b=-1$. In both cases, we have $\abs{\wvec{j}_{T+1} \cdot \pp_l} < \eta$.
	Furthermore, by \eqref{eq:gradient_13}, $\wvec{j}_{T+1} \cdot \pp_i \ge \wvec{j}_{T} \cdot \pp_i > \eta$. Hence, $\argmax_{1 \le l \le 4} \wvec{j}_{T+1}\cdot \pp_l = i$ which by definition implies that $j \in  W_{T+1}^+(i)$. This concludes the proof by induction which shows that $W_0^+(i) \subseteq W_{t}^+(i)$ for all $t \ge 1$.
	
	In order to prove the lemma, it suffices to show that $W_0^+(2) \cup W_0^+(4) \subseteq W_t^+(2) \cup W_t^+(4)$ and prove the second claim. This follows since $\bigcup_{i=1}^4{W_t^+(i)} = \{1,2,...,k\}$.
	We will show by induction on $t \ge 1$, that for all $j \in W_0^+(2) \cup W_0^+(4)$, the following holds:
	\begin{enumerate}
		\item $j \in W_{t}^+(2)\cap W_{t}^+(4)$.
		\item $\wvec{j}_{t} = \wvec{j}_{0} + m \pp_2$ for $m \in \mathbb{Z}$.
	\end{enumerate}
	The claim holds for $t=1$ by the proof of Lemma \ref{lem:iteration_2}. Assume it holds for $t=T$. By the induction hypothesis $j \in W_{T}^+(2)\cap W_{T}^+(4)$. Assume without loss of generality that $j \in W_{T}^+(2)$. This implies that $j \in W_{T}^-(2)$ as well. Therefore, by \eqref{eq:w_gradient} we have
	\begin{equation}
		\label{eq:gradient_2}
		\wvec{j}_{T+1} = \wvec{j}_{T} + a\eta\pp_2 + b\eta\pp_{2}
	\end{equation}
	where $a \in \{0,1\}$ and $b \in \{0,-1\}$. By the induction hypothesis, $\wvec{j}_{T+1} = \wvec{j}_{0} + m \pp_2$ for $m \in \mathbb{Z}$. If $a = 1$ or $b = 0$ we have for $i \in \{1,3\}$, $$\wvec{j}_{T+1} \cdot \pp_2 \ge \wvec{j}_{T} \cdot \pp_2 > \wvec{j}_{T} \cdot \pp_i = \wvec{j}_{T+1} \cdot \pp_i$$
	where the first inequality follows since $j \in W_{T}^+(2)$ and the second by \eqref{eq:gradient_2}.
	This implies that $j \in W_{T+1}^+(2)\cap W_{T+1}^+(4)$.
	
	Otherwise, assume that $a = 0$ and $b = -1$. By Lemma \ref{lem:init_bound_disj} we have $\wvec{j}_{0} \cdot \pp_2 < \frac{\sqrt{2}\eta}{4}$. Since $j \in W_{T}^+(2)$, it follows by the induction hypothesis that $\wvec{j}_{T} = \wvec{j}_{0} + m  \pp_2$, where $m \in \mathbb{Z}$ and $m \ge 0$. To see this, note that if $m < 0$, then $\wvec{j}_{T} \cdot \pp_2 < 0$ and $j \notin W_{T}^+(2)$, which is a contradiction. Let $i \in \{1,3\}$. If $m=0$, then $\wvec{j}_{T+1} = \wvec{j}_{0} - \pp_2$, $\wvec{j}_{T+1} \cdot \pp_4 > \frac{\eta}{2}$ and $\wvec{j}_{T+1} \cdot \pp_i = \wvec{j}_{0} \cdot \pp_i < \frac{\sqrt{2}\eta}{4}$ by Lemma \ref{lem:init_bound_disj}. Therefore, $j \in W_{T+1}^+(4)$. 
	
	Otherwise, if $m > 0$, then $\wvec{j}_{T+1} \cdot \pp_2 \ge \wvec{j}_{0} \cdot \pp_2 >  \wvec{j}_{0} \cdot \pp_i = \wvec{j}_{T+1} \cdot \pp_i$. Hence, $j \in W_{T+1}^+(2)$, which concludes the proof.
\end{proof}

\begin{lem}
	\label{lem:u24}
	For all $t \ge 0$ we have 
	\begin{enumerate}
		\item $\uvec{j}_t = \uvec{j}_0 + m\eta\pp_2$ for $m \in \mathbb{Z}$.
		\item $U_0^+(2)\cup U_0^+(4) \subseteq U_t^+(2)\cup U_t^+(4)$. 
	\end{enumerate}
	.
\end{lem}
\begin{proof}
	Let $j \in U_0^+(2)\cup U_0^+(4)$. It suffices to prove that $\uvec{j}_t = \uvec{j}_0 + \alpha_t\eta\pp_2$ for $\alpha_t \in \mathbb{Z}$. This follows since the inequalities $\abs{\uvec{j}_0 \cdot \pp_1} < \abs{\uvec{j}_0 \cdot \pp_2} \le \frac{\sqrt{2}\eta}{4}$ imply that in this case $j \in U_t^+(2)\cup U_t^+(4)$. Assume by contradiction that there exist an iteration $t$ for which $\uvec{j}_t = \uvec{j}_0 + \alpha_t\eta\pp_2 + \beta_t\eta\pp_i$ where $\beta_t \in \{-1,1\}$, $\alpha_t \in \mathbb{Z}$, $i \in \{1,3\}$ and $\uvec{j}_{t-1} = \uvec{j}_0 + \alpha_{t-1}\eta\pp_2$ where $\alpha_{t-1} \in \mathbb{Z}$. \footnote{Note that in each iteration $\beta_t$ changes by at most $\eta$.} Since the coefficient of $\pp_i$ changed in iteration $t$, we have $j \in U_{t-1}^+(1)\cup U_{t-1}^+(3)$. However, this contradicts the claim above which shows that if $\uvec{j}_{t-1} = \uvec{j}_0 + \alpha_{t-1}\eta\pp_2$, then $j \in U_{t-1}^+(2)\cup U_{t-1}^+(4)$.
\end{proof}

\begin{lem}
	\label{lem:u13}
	Let $i \in \{1,3\}$ and $l \in \{2,4\}$. For all $t \ge 0$, if $j \in U_0^+(i) \cap U_0^-(l)$, then there exists $a_t \in \{0,-1\}$, $b_t \in \mathbb{N}$ such that $\uvec{j}_t = \uvec{j}_0 + a_t\eta\pp_i + b_t\eta\pp_l$.
\end{lem}
\begin{proof}
	First note that by \eqref{eq:u_gradient} we generally have $\uvec{j}_t = \uvec{j}_0 + \alpha\eta\pp_i + \beta\eta\pp_l$ where $\alpha, \beta \in \mathbb{Z}$. Since $\abs{\uvec{j}_0 \cdot \pp_1} \le \frac{\sqrt{2}\eta}{4}$, by the gradient update in \eqref{eq:u_gradient} it holds that $a_t \in \{0,-1\}$. Indeed, $a_0=0$ and by the gradient update if $a_{t-1} = 0$ or $a_{t-1} = -1$ then $a_{t} \in \{-1,0\}$.
	
	Assume by contradiction that there exists an iteration $t > 0$ such that $b_t = -1$ and $b_{t-1} = 0$. Note that by \eqref{eq:u_gradient} this can only occur if $j \in U_{t-1}^+(l)$. We have $\uvec{j}_{t-1} = \uvec{j}_0 + a_{t-1} \eta \pp_i$ where $a_{t-1} \in \{0,-1\}$. Observe that $\abs{\uvec{j}_{t-1} \cdot \pp_i} \ge \abs{\uvec{j}_{0} \cdot \pp_i}$ by the fact that $\abs{\uvec{j}_{0} \cdot \pp_i} \le \frac{\sqrt{2}\eta}{4}$. Since $\uvec{j}_{0} \cdot \pp_i > \uvec{j}_{0} \cdot \pp_l = \uvec{j}_{t-1} \cdot \pp_l$ we have $j \in U_{t-1}^+(1)\cup U_{t-1}^+(3)$, a contradiction.
\end{proof}

\subsubsection{Bounding $P_t^+$, $P_t^-$ and $S_t^-$}
\label{sec:bounds1}

\begin{lem}
	\label{lem:small_sums}
	The following holds
	\begin{enumerate}
		\item $S_t^- \le \abs{W_t^+(1) \cup W_t^+(3)}\eta$ for all $t \ge 1$.
		\item $P_t^+ \le \abs{U_t^+(1) \cup U_t^+(3)}\eta$ for all $t \ge 1$.
		\item $P_t^- \le \abs{U_t^+(1) \cup U_t^+(3)}\eta$ for all $t \ge 1$.
	\end{enumerate}
\end{lem}
\begin{proof}
	In Lemma \ref{lem:same_w} we showed that for all $t \ge 0$ and $j \in W_t^+(1) \cup W_t^+(3)$ it holds that $\abs{\wvec{j}_t \cdot \pp_2} \le \eta$ . This proves the first claim. The second claim follows similarly. Without loss of generality, let $j \in U_t^+(1)$. By Lemma \ref{lem:u24} it holds that $U_{t'}^+(1) \subseteq U_0^+(1) \cup U_0^+(3)$ for all $t' \le t$. Therefore, by Lemma \ref{lem:u13} we have $\abs{\uvec{j}_t\pp_1} < \eta$, from which the claim follows.
	
	For the third claim, without loss of generality, assume by contradiction that for $j \in U_t^+(1)$ it holds that $\abs{\uvec{j}_t \cdot \pp_2} > \eta$. Since $\abs{\uvec{j}_t \cdot \pp_1} < \eta$ by Lemma \ref{lem:u13}, it follows that $j \in U_t^+(2)\cup U_t^+(4)$, a contradiction. Therefore, $\abs{\uvec{j}_t \cdot \pp_2} \le \eta$ for all $j \in U_t^+(1) \cup U_t^+(3)$, from which the claim follows.
\end{proof}

\subsubsection{Dynamics of $S_t^+$}
\label{sec:dynamics}

\begin{lem}
	\label{lem:proportion_13}
	Let $$X_t^+ =  \sum_{j \in W_t^+(1)}\left[\max\left\{\sigma\left(\wvec{i}\cdot \xx_1^+ \right),...,\sigma\left(\wvec{i}\cdot \xx_d^+ \right)\right\}\right] $$
	and 
	$$Y_t^+ =  \sum_{j \in W_t^+(3)}\left[\max\left\{\sigma\left(\wvec{i}\cdot \xx_1^+ \right),...,\sigma\left(\wvec{i}\cdot \xx_d^+ \right)\right\}\right] $$
	Then for all $t$, $\frac{X_{t}^+ - X_0^+}{\abs{W_{t}^+(1)}} = \frac{Y_{t}^+ - Y_0^+}{\abs{W_{t}^+(3)}}$.
\end{lem}
\begin{proof}
	We will prove the claim by induction on $t$. For $t=0$ this clearly holds. Assume it holds for $t = T$. Let $j_1 \in W_T^+(1)$ and $j_2 \in W_T^+(3)$. By \eqref{eq:w_gradient}, the gradient updates of the corresponding weight vector are given as follows:
	$$\wvec{j_1}_{T+1} = \wvec{j_1}_{T} + a \eta \pp_1 + b_1 \eta \pp_2$$
	and 
	$$\wvec{j_2}_{T+1} = \wvec{j_2}_{T} + a \eta \pp_3 + b_2 \eta \pp_2$$
	where $a \in \{0,1\}$ and $b_1,b_2 \in \{-1,0,1\}$.
	By Lemma \ref{lem:same_w}, $j_1 \in W_{T+1}^+(1)$ and $j_2 \in W_{T+1}^+(3)$. Therefore, $$\max\left\{\sigma\left(\wvec{j_1}_{T+1}\cdot \xx_1^+ \right),...,\sigma\left(\wvec{j_1}_{T+1}\cdot \xx_d^+ \right)\right\} = \max\left\{\sigma\left(\wvec{j_1}_{T}\cdot \xx_1^+ \right),...,\sigma\left(\wvec{j_1}_{T}\cdot \xx_d^+ \right)\right\} + a\eta$$
	and 
	$$\max\left\{\sigma\left(\wvec{j_2}_{T+1}\cdot \xx_1^+ \right),...,\sigma\left(\wvec{j_2}_{T+1}\cdot \xx_d^+ \right)\right\} = \max\left\{\sigma\left(\wvec{j_2}_{T}\cdot \xx_1^+ \right),...,\sigma\left(\wvec{j_2}_{T}\cdot \xx_d^+ \right)\right\} + a\eta$$
	By Lemma $\ref{lem:same_w}$ we have $\abs{W_t^+(1)} = \abs{W_0^+(1)}$ and $\abs{W_t^+(3)} = \abs{W_0^+(3)}$ for all $t$. It follows that 
	\begin{align*}
		\frac{X_{T+1}^+ - X_0^+}{\abs{W_{T+1}^+(1)}} &= \frac{a\eta\abs{W_{0}^+(1)} + X_{T}^+ - X_0^+}{\abs{W_{0}^+(1)}} \\ &= a\eta + \frac{Y_{T}^+ - Y_0^+}{\abs{W_{0}^+(3)}} \\ &= \frac{a\eta\abs{W_{0}^+(3)} + Y_{T}^+ - Y_0^+}{\abs{W_{0}^+(3)}} \\ &= \frac{Y_{T+1}^+ - Y_0^+}{\abs{W_{T+1}^+(3)}}
	\end{align*}
	where the second equality follows by the induction hypothesis. This proves the claim.
\end{proof}

\begin{lem}
	\label{lem:dynamics}
	The following holds:
	\begin{enumerate}
		\item If $\nett{t}(\xx^+) < \gamma$ and $-\nett{t}(\xx^-) < 1$, then 
		$S_{t+1}^+ = S_{t}^+ + \eta\abs{W_t^+(1) \cup W_t^+(3)}$.
		\item If $\nett{t}(\xx^+) \ge \gamma$ and $-\nett{t}(\xx^-) < 1$, then $S_{t+1}^+ = S_t^+$.
		\item If $\nett{t}(\xx^+) < \gamma$ and $-\nett{t}(\xx^-) \ge 1$, then $S_{t+1}^+ = S_{t}^+ + \eta\abs{W_t^+(1) \cup W_t^+(3)}$.
		
	\end{enumerate}
\end{lem}
\begin{proof}
	\begin{enumerate}
		\item The equality follows since for each $i \in \{1,3\}$, $l \in \{2,4\}$ and $j \in W_t^+(i) \cap W_t^-(l)$ we have $\wvec{j}_{t+1} = \wvec{j}_{t} + \eta \pp_i - \eta \pp_l$ and $W_{t+1}^+(1) \cup W_{t+1}^+(3) = W_{t}^+(1) \cup W_{t}^+(3)$ by Lemma \ref{lem:same_w}.
		\item In this case for each $i \in \{1,3\}$, $l \in \{2,4\}$ and $j \in W_t^+(i) \cap W_t^-(l)$ we have $\wvec{j}_{t+1} = \wvec{j}_{t} - \eta \pp_l$ and $W_{t+1}^+(1) \cup W_{t+1}^+(3) = W_{t}^+(1) \cup W_{t}^+(3)$ by Lemma \ref{lem:same_w}.
		\item This equality follows since for each $i \in \{1,3\}$, $l \in \{2,4\}$ and $j \in W_t^+(i) \cap W_t^-(l)$ we have $\wvec{j}_{t+1} = \wvec{j}_{t} + \eta \pp_i$ and $W_{t+1}^+(1) \cup W_{t+1}^+(3) = W_{t}^+(1) \cup W_{t}^+(3)$ by Lemma \ref{lem:same_w}.
	\end{enumerate}

\end{proof}

\subsubsection{Upper Bounds on $\nett{t}(\xx^+)$, $-\nett{t}(\xx^-)$ and $S_t^+$}
\label{sec:bounds2}

\begin{lem}
	\label{lem:phase2_bound}
	Assume that $\nett{t}(\xx^+) \ge \gamma$ and $-\nett{t}(\xx^-) < 1$ for $T \le t < T+b$ where $b \ge 2$. Then $\nett{T+b}(\xx^+) \le \nett{T}(\xx^+) - (b-1)\ceta + \eta\abs{W_{0}^+(2) \cup W_{0}^+(4)}$.
\end{lem}
\begin{proof}
	Define $R_t^+=Y_t^+ - Z_t^+$ where $$Y_t^+ =  \sum_{j \in W_t^+(2) \cup W_t^+(4)}\left[\max\left\{\sigma\left(\wvec{i}\cdot \xx_1^+ \right),...,\sigma\left(\wvec{i}\cdot \xx_d^+ \right)\right\}\right] $$
	and 
	$$Z_t^+ = \sum_{j \in U_t^+(2) \cup U_t^+(4)}\left[\max\left\{\sigma\left(\uvec{i}\cdot \xx_1^+ \right),...,\sigma\left(\uvec{i}\cdot \xx_d^+ \right)\right\}\right] $$

	Let $l \in \{2,4\}$, $t=T$ and $j \in U_{t+1}^+(l)$. Then, either $j \in U_{t}^+(2) \cup U_{t}^+(4)$ or $j \in U_{t}^+(1) \cup U_{t}^+(3)$. In the first case, $\uvec{j}_{t+1} = \uvec{j}_{t} + \eta \pp_l$. Note that this implies that $ U_{t}^+(2) \cup U_{t}^+(4) \subseteq U_{t+1}^+(2) \cup U_{t+1}^+(4)$ (since $\pp_l$ will remain the maximal direction). Therefore,
	\begin{align*}
		&\sum_{j \in \left(U_{t+1}^+(2) \cup U_{t+1}^+(4)\right)\bigcap\left(U_{t}^+(2) \cup U_{t}^+(4)\right)}\left[\max\left\{\sigma\left(\uvec{j}_{t+1}\cdot \xx_1^+ \right),...,\sigma\left(\uvec{j}_{t+1}\cdot \xx_d^+ \right)\right\}\right] \\ &-  \sum_{j \in U_t^+(2) \cup U_t^+(4)}\left[\max\left\{\sigma\left(\uvec{j}_t\cdot \xx_1^+ \right),...,\sigma\left(\uvec{j}_{t+1}\cdot \xx_d^+ \right)\right\}\right] \\   &= \eta\abs{\left(U_{t+1}^+(2) \cup U_{t+1}^+(4)\right)\bigcap\left(U_{t}^+(2) \cup U_{t}^+(4)\right)}  \\ &= \eta\abs{U_{t}^+(2) \cup U_{t}^+(4)} \numberthis \label{eq:zplus1}
	\end{align*}
	
	In the second case, where we have $j \in U_{t}^+(1) \cup U_{t}^+(3)$, it holds that $\uvec{j}_{t+1} = \uvec{j}_{t} + \eta \pp_l$, $j \in U_{t}^-(l)$ and $\uvec{j}_{t+1} \cdot \pp_l > \eta$. Furthermore, by Lemma \ref{lem:u13}, $\uvec{j}_{t} \cdot \pp_i < \eta$ for $i \in \{1,3\}$. Note that by Lemma \ref{lem:u13}, any $j_1 \in U_{t}^+(1) \cup U_{t}^+(3) $ satisfies $j_1 \in U_{t+1}^+(2) \cup U_{t+1}^+(4) $. By all these observations, we have
	
	\begin{align*}
		&\sum_{j \in \left(U_{t+1}^+(2) \cup U_{t+1}^+(4)\right)\bigcap\left(U_{t}^+(1) \cup U_{t}^+(3)\right)}\left[\max\left\{\sigma\left(\uvec{j}_{t+1}\cdot \xx_1^+ \right),...,\sigma\left(\uvec{j}_{t+1}\cdot \xx_d^+ \right)\right\}\right] \\ &-  \sum_{j \in U_t^+(1) \cup U_t^+(3)}\left[\max\left\{\sigma\left(\uvec{j}_t\cdot \xx_1^+ \right),...,\sigma\left(\uvec{j}_{t+1}\cdot \xx_d^+ \right)\right\}\right] \\ &\ge 0 \numberthis \label{eq:zplus2}
	\end{align*}
	
	By \eqref{eq:zplus1} and \eqref{eq:zplus2}, it follows that, $Z_{t+1}^+ + P_{t+1}^+ \ge Z_{t+1}^+ \ge Z_t^+ + P_t^+ + \eta \abs{U_{t}^+(2) \cup U_{t}^+(4)}$. By induction we have $Z_{t+b}^+ + P_{t+b}^+ \ge Z_t^+ + P_{t}^+ +  \sum_{i=0}^{b-1}\eta \abs{U_{t+i}^+(2) \cup U_{t+i}^+(4)}$. By Lemma \ref{lem:u13} for any $1 \le i \le b-1$ we have $\abs{U_{t+i}^+(2) \cup U_{t+i}^+(4)} = \{1,...,k\}$. Therefore, $Z_{t+b}^+ + P_{t+b}^+ \ge Z_t^+ + P_{t}^+ + (b-1)\ceta$. 
	
	Now, assume that $j \in W_T^+(l)$ for $l \in \{2,4\}$. Then $\wvec{j}_{T+1} = \wvec{j}_{T} - \eta \pp_l$. 
	Thus either  $$\max\left\{\sigma\left(\wvec{j}_{T+1}\cdot \xx_1^+ \right),...,\sigma\left(\wvec{j}_{T+1}\cdot \xx_d^+ \right)\right\} - \max\left\{\sigma\left(\wvec{j}_{T}\cdot \xx_1^+ \right),...,\sigma\left(\wvec{j}_{T}\cdot \xx_d^+ \right)\right\} = -\eta $$
	in the case that $j \in W_{T+1}^+(l)$, or 
	$$\max\left\{\sigma\left(\wvec{j}_{T+1}\cdot \xx_1^+ \right),...,\sigma\left(\wvec{j}_{T+1}\cdot \xx_d^+ \right)\right\} \le \eta$$
	if $j \notin W_{T+1}^+(l)$.
	
	Applying these observations $b$ times, we see that $Y_{T+b}^+ - Y_{T}^+$ is at most $\eta \abs{W_{T+b}^+(2) \cup W_{T+b}^+(4)} =\eta \abs{W_{0}^+(2) \cup W_{0}^+(4)}$ where the equality follows by Lemma \ref{lem:same_w}. By Lemma \ref{lem:dynamics}, we have $S_{T+b}^+ = S_{T}^+$.
	
	Hence, we can conclude that 
	\begin{align*}
		\nett{T+b}(\xx^+) - \nett{T}(\xx^+) &= S_{T+b}^+ + R_{T+b}^+ - P_{T+b}^+ - S_{T}^- - R_{T}^+ + P_{T}^+ \\ &=  Y_{T+b}^+ - Z_{T+b}^+ - P_{T+b}^+   - Y_{T}^+ + Z_{T}^+ + P_{T}^+\\ &\le - (b-1)\ceta + \eta\abs{W_{0}^+(2) \cup W_{0}^+(4)}
	\end{align*}

\end{proof}

\begin{lem}
	\label{lem:phase3_bound}
	Assume that $\nett{t}(\xx^+) < \gamma$ and $-\nett{t}(\xx^-) \ge 1$ for $T \le t < T+b$ where $b \ge 1$. Then $-\nett{T+b}(\xx^-) \le -\nett{T}(\xx^-) - b\eta \abs{W_{0}^+(2) \cup W_{0}^+(4)} + \ceta$.
\end{lem}
\begin{proof}
	Define $$Y_t^- =  \sum_{j \in W_t^+(2) \cup W_t^+(4)}\left[\max\left\{\sigma\left(\wvec{i}\cdot \xx_1^+ \right),...,\sigma\left(\wvec{i}\cdot \xx_d^+ \right)\right\}\right] $$
	and 
	$$Z_t^- = \sum_{j = 1}^k\left[\max\left\{\sigma\left(\uvec{j}\cdot \xx_1^+ \right),...,\sigma\left(\uvec{j}\cdot \xx_d^+ \right)\right\}\right] $$
	
	First note that by Lemma \ref{lem:same_w} we have $W_{t+1}^+(2) \cup W_{t+1}^+(4) = W_t^+(2) \cup W_t^+(4)$. Next, for any $l \in \{2,4\}$ and $j \in W_t^+(l)$ we have $\wvec{j}_{t+1} = \wvec{j}_{t} + \eta \pp_l$. Therefore, 
	$$Y_{T+b}^- \ge Y_T^- + b \eta \abs{W_{T}^+(2) \cup W_{T}^+(4)} = Y_T^- + b \eta \abs{W_{0}^+(2) \cup W_{0}^+(4)} $$
	where the second equality follows by Lemma \ref{lem:same_w}.

	Assume that $j \in U_T^+(l)$ for $l \in \{1,3\}$. Then $\uvec{j}_{T+1} = \uvec{j}_{T} - \eta \pp_l$ and 
	\begin{equation}
		\label{eq:u13_phase3}
		\max\left\{\sigma\left(\uvec{j}_{T+1}\cdot \xx_1^- \right),...,\sigma\left(\uvec{j}_{T+1}\cdot \xx_d^- \right)\right\} - \max\left\{\sigma\left(\uvec{j}_{T}\cdot \xx_1^- \right),...,\sigma\left(\uvec{j}_{T}\cdot \xx_d^- \right)\right\} = 0
	\end{equation}
	To see this, note that by Lemma \ref{lem:u13} and Lemma \ref{lem:u24} it holds that $\uvec{j}_T =  \uvec{j}_0 + a_T\eta\pp_l$ where $a_T \in \{-1,0\}$. Hence, $\uvec{j}_{T+1} =  \uvec{j}_0 + a_{T+1}\eta\pp_l$ where $a_{T+1} \in \{-1,0\}$. Since $\abs{\uvec{j}_0 \cdot \pp_2} < \frac{\sqrt{2}\eta}{4}$ it follows that $\uvec{j}_{T+1} \cdot \pp_2 = \uvec{j}_{T} \cdot \pp_2 = \uvec{j}_{0} \cdot \pp_2$ and thus \eqref{eq:u13_phase3} holds.
	
	Now assume that $j \in U_T^+(l)$ for $l \in \{2,4\}$. Then
	$$\max\left\{\sigma\left(\uvec{j}_{T+1}\cdot \xx_1^- \right),...,\sigma\left(\uvec{j}_{T+1}\cdot \xx_d^- \right)\right\} - \max\left\{\sigma\left(\uvec{j}_{T}\cdot \xx_1^- \right),...,\sigma\left(\uvec{j}_{T}\cdot \xx_d^- \right)\right\} = -\eta $$
	if $l\in \{2,4\}$ and $j \in U_{T+1}^+(l)$, or 
	$$\max\left\{\sigma\left(\uvec{j}_{T+1}\cdot \xx_1^- \right),...,\sigma\left(\uvec{j}_{T+1}\cdot \xx_d^- \right)\right\} \le \eta$$
	if $l\in \{2,4\}$ and $j \notin U_{T+1}^+(l)$.
	
	Applying these observations $b$ times, we see that $Z_{T+b}^- - Z_{T}^-$ is at most $\eta \abs{U_{T+b}^+(2) \cup U_{T+b}^+(4)}$. 
	Furthermore, for $j \in W_{T}^+(l)$, $l \in \{1,3\}$, it holds that $\wvec{j}_{T+1} = \wvec{j}_{T} + \eta \pp_l$. Therefore
	$$\max\left\{\sigma\left(\wvec{j}_{T+1}\cdot \xx_1^- \right),...,\sigma\left(\wvec{j}_{T+1}\cdot \xx_d^- \right)\right\} = \max\left\{\sigma\left(\wvec{j}_{T}\cdot \xx_1^- \right),...,\sigma\left(\wvec{j}_{T}\cdot \xx_d^- \right)\right\} $$ and since $W_{T+1}^+(1) \cup W_{T+1}^+(3)=W_{T}^+(1) \cup W_{T}^+(3)$ by Lemma \ref{lem:same_w}, we get $S_{T+b}^- = S_{T}^-$.
	Hence, we can conclude that 
	\begin{align*}
		-\nett{T+b}(\xx^-) + \nett{T}(\xx^-) &= -S_{T+b}^- - Y_{T+b}^- + Z_{T+b}^- + S_{T}^- + Y_{T}^- - Z_{T}^- \\ &\le - b\eta \abs{W_{0}^+(2) \cup W_{0}^+(4)} + \eta \abs{U_{T+b}^+(2) \cup U_{T+b}^+(4)} \\ &\le - b\eta \abs{W_{0}^+(2) \cup W_{0}^+(4)} + \ceta
	\end{align*}

\end{proof}
\begin{lem}
	\label{lem:function_value_bound}
	For all $t$, $\nett{t}(\xx^+) \le \gamma + 3\ceta$, $-\nett{t}(\xx^-) \le 1 + 3\ceta$ and $S_t^+ \le \gamma + 1 +  8 \ceta$.
\end{lem}
\begin{proof}
	The claim holds for $t=0$. Consider an iteration $T$. If $\nett{T}(\xx^+) < \gamma$ then $\nett{T+1}(\xx^+) \le \nett{T}(\xx^+) + 2\eta k \le \gamma + 2\ceta$. Now assume that $\nett{t}(\xx^+) \ge \gamma$ for $T \le t \le T+b$ and $\nett{T-1}(\xx^+) < \gamma$. By Lemma \ref{lem:phase2_bound}, it holds that $\nett{T+b}(\xx^+) \le \nett{T}(\xx^+) + \eta k \le \nett{T}(\xx^+) + \ceta \le \gamma + 3\ceta$, where the last inequality follows from the previous observation. Hence, $\nett{t}(\xx^+) \le \gamma + 3\ceta$ for all $t$.
	
	The proof of the second claim follows similarly. It holds that $-\nett{T+1}(\xx^-) < 1 + 2\ceta$ if $-\nett{T}(\xx^-) < 1$. Otherwise if $-\nett{t}(\xx^-) \ge 1$ for $T \le t \le T+b$ and $-\nett{T-1}(\xx^-) < 1$ then $-\nett{T+b}(\xx^-) \le 1+ 3\ceta$ by Lemma \ref{lem:phase3_bound}.
	
	The third claim holds by the following identities and bounds $\nett{T}(\xx^+)-\nett{T}(\xx^-) = S_T^+ - P_T^+ +  P_T^- - S_T^-$, $P_T^- \ge 0$,  $\abs{P_T^+} \le \ceta$, $\abs{S_T^-}\le \ceta$ and $\nett{T}(\xx^+)-\nett{T}(\xx^-) \le \gamma + 1 + 6\ceta$ by the previous claims.
\end{proof}

\subsubsection{Optimization}
\label{sec:opt}

We are now ready to prove a global optimality guarantee for gradient descent.

\begin{prop}
	\label{prop:opt}
	Let $k > 16$ and $\gamma \ge 1$. With probabaility at least $1-\frac{\sqrt{2k}}{\sqrt{\pi} e^{8k}} - 4e^{-8}$, after $T=\frac{7(\gamma + 1 + 8\ceta)}{\left(\frac{k}{2} - 2\sqrt{k}\right)\eta}$ iterations, gradient descent converges to a global minimum.
\end{prop}
\begin{proof}
	First note that with probability at least $1-\frac{\sqrt{2k}}{\sqrt{\pi} e^{8k}} - 4e^{-8}$ the claims of Lemma \ref{lem:init_num_disj} and Lemma \ref{lem:init_bound_disj} hold. Now, if gradient descent has not reached a global minimum at iteration $t$ then either $\nett{t}(\xx^+) < \gamma$ or $-\nett{t}(\xx^-) < 1$. If $-\nett{t}(\xx^+) < \gamma$ then by Lemma \ref{lem:dynamics} it holds that 
	\begin{equation}
		\label{eq:opt_increase}
		S_{t+1}^+ \ge S_t^+ + \eta\abs{W_0^+(1) \cup W_0^+(3)} \ge S_t^+ + \left(\frac{k}{2} - 2\sqrt{k}\right)\eta
	\end{equation}
	where the last inequality follows by Lemma \ref{lem:init_num_disj}. 
	
	If $\nett{t}(\xx^+) \ge \gamma$ and $-\nett{t}(\xx^-) < 1$ we have $S_{t+1}^+=S_t^+$ by Lemma \ref{lem:dynamics}. However, by Lemma \ref{lem:phase2_bound}, it follows that after 5 consecutive iterations $t < t' < t+6$ in which $\nett{t'}(\xx^+) \ge \gamma$ and $-\nett{t'}(\xx^-) < 1$, we have $\nett{t+6}(\xx^+) < \gamma$.  To see this, first note that for all $t$, $\nett{t}(\xx^+) \le \gamma + 3\ceta$ by Lemma \ref{lem:function_value_bound}. Then, by Lemma \ref{lem:phase2_bound}  we have 
	\begin{align*}
		\nett{t+6}(\xx^+) &\le \nett{t}(\xx^+) - 5\ceta  + \eta\abs{W_{0}^+(2) \cup W_{0}^+(4)} \\ &\le \gamma + 3\ceta - 5\ceta + \ceta \\ &< \gamma
	\end{align*}
	where the second inequality follows by Lemma \ref{lem:init_num_disj} and the last inequality by the assumption on $k$.
	
	Assume by contradiction that GD has not converged to a global minimum after $T=\frac{7(\gamma + 1 + 8\ceta)}{\left(\frac{k}{2} - 2\sqrt{k}\right)\eta}$ iterations. Then, by the above observations, and the fact that $S_0^+ > 0$ with probability $1$, we have
	\begin{align*}
		S_T^+ &\ge S_0^+ + \left(\frac{k}{2} - 2\sqrt{k}\right)\eta \frac{T}{7}\\ &> \gamma + 1 + 8\ceta
	\end{align*}
	However, this contradicts Lemma \ref{lem:function_value_bound}.
\end{proof}

\subsubsection{Generalization on Positive Class}
\label{sec:positive_class}

We will first need the following three lemmas.

\begin{lem}
	\label{lem:init_num_disj_W1W3}
	With probability at least $1-4e^{-8}$, it holds that $$\abs{\abs{W_0^+(1)} - \frac{k}{4}}\le 2\sqrt{k}$$ and $$\abs{\abs{W_0^+(3)} - \frac{k}{4}}\le 2\sqrt{k}$$
\end{lem}
\begin{proof}
	The proof is similar to the proof of Lemma \ref{lem:init_num_disj}.
\end{proof}

\begin{lem}
	\label{lem:stplus_lower_bound}
	Assume that gradient descent converged to a global minimum at iteration $T$. Then there exists an iteration $T_2 < T$ for which $S_{t}^+ \ge \gamma + 1 - 3 \ceta$ for all $t \ge T_2$ and for all $t < T_2$, $-\nett{t}(\xx^-) < 1$.
\end{lem}
\begin{proof}
	Assume that for all $0 \le t \le T_1$ it holds that $\nett{t}(\xx^+) < \gamma$ and $-\nett{t}(\xx^-) < 1$. By continuing the calculation of Lemma \ref{lem:iteration_2} we have the following:
	\begin{enumerate}
		\item For $i \in \{1,3\}$, $l \in \{2,4\}$, $j \in W_0^+(i) \cap W_0^-(l)$, it holds that $\wvec{j}_{T_1} = \wvec{j}_0 + T_1\eta \pp_i- \frac{1}{2}(1-(-1)^{T_1})\eta\pp_l$ .
		\item For $i \in \{2,4\} $ and $j \in W_0^+(i)$, it holds that $\wvec{j}_{T_1} = \wvec{j}_0$.
		\item For $i \in \{1,3\}$, $l \in \{2,4\}$, $j \in U_0^+(i) \cap U_0^-(l)$, it holds that $\uvec{j}_{T_1} = \uvec{j}_0  -\eta\pp_i + \eta \pp_l$.
		\item For $i \in \{2,4\} $ and $j \in U_0^+(i)$, it holds that $\uvec{j}_{T_1} = \uvec{j}_0$.
	\end{enumerate}
	
	Therefore, there exists an iteration $T_1$ such that $\nett{T_1}(\xx^+) \ge \gamma$ and $-\nett{T_1}(\xx^-) < 1$ and for all $t<T_1$, $\nett{t}(\xx^+) < \gamma$ and $-\nett{t}(\xx^-) < 1$. 
	Let $T_2 \le T$ be the first iteration such that $-\nett{T_2}(\xx^-) \ge 1$. We claim that for all $T_1 \le t \le T_2$ we have $\nett{T_1}(\xx^+) \ge \gamma - 2\ceta$. It suffices to show that for all $T_1 \le t < T_2$ the following holds:
	\begin{enumerate}
		\item If $\nett{t}(\xx^+) \ge \gamma$ then $\nett{t+1}(\xx^+) \ge \gamma -2\ceta$.
		\item If $\nett{t}(\xx^+) < \gamma$ then $\nett{t+1}(\xx^+) \ge \nett{t}(\xx^+)$.
	\end{enumerate}
	
	The first claim follows since at any iteration $\nett{t}(\xx^+)$ can decrease by at most $2\eta k = 2\ceta$. For the second claim, let $t' < t$ be the latest iteration such that $\nett{t'}(\xx^+) \ge \gamma$. Then at iteration $t'$ it holds that  $-\nett{t'}(\xx^-) < 1$ and $\nett{t'}(\xx^+) \ge \gamma$. Therefore, for all $i \in \{1,3\}$, $l \in \{2,4\}$ and $j \in U_0^+(i) \cap U_0^+(l)$ it holds that $\uvec{j}_{t'+1} = \uvec{j}_{t'} + \eta\pp_l$. Hence, by Lemma \ref{lem:u24} and Lemma \ref{lem:u13} it holds that $U_{t'+1}^+(1) \cup U_{t'+1}^+(3) = \emptyset$. Therefore, by the gradient update in \eqref{eq:u_gradient}, 
	for all $1 \le j \le k$, and all $t' < t'' \le t$ we have $\uvec{j}_{t''+1} = \uvec{j}_{t''}$, which implies that $\nett{t''+1}(\xx^+) \ge \nett{t''}(\xx^+)$. For $t'' = t$ we get $\nett{t+1}(\xx^+) \ge \nett{t}(\xx^+)$.
	
	The above argument shows that $\nett{T_2}(\xx^+) \ge \gamma -2\ceta$ and $-\nett{T_2}(\xx^-) \ge 1$.
	Since  $\nett{T_2}(\xx^+)-\nett{T_2}(\xx^-) = S_{T_2}^+ - P_{T_2}^+ +  P_{T_2}^- - S_{T_2}^-$, $P_{T_2}^-,S_{T_2}^- \ge 0$ and  $\abs{P_{T_2}^-} \le \ceta$ it follows that $S_{T_2}^+ \ge \gamma + 1 - 3\ceta$. Finally, by Lemma \ref{lem:dynamics} we have $S_{t}^+ \ge \gamma + 1 - 3\ceta$ for all $t \ge T_2$.
\end{proof}

\begin{lem}
	\label{lem:bounds24}
	Let $$X^+_t = \sum_{j \in W_t^+(2)\cup W_t^+(4)}\left[\max\left\{\sigma\left(\wvec{j}\cdot \xx_1^+ \right),...,\sigma\left(\wvec{j}\cdot \xx_d^+ \right)\right\}\right]$$
	and 
	$$Y^+_t = \sum_{j \in U_t^+(2)\cup U_t^+(4)}\left[\max\left\{\sigma\left(\uvec{j}\cdot \xx_1^+ \right),...,\sigma\left(\uvec{j}\cdot \xx_d^+ \right)\right\}\right]$$
	Assume that $k \ge 64$ and gradient descent converged to a global minimum at iteration $T$. Then, $X^+_T \le 34\ceta$ and $Y^+_T \le 1 + 38\ceta$.
\end{lem}
\begin{proof}
	Notice that by the gradient update in \eqref{eq:w_gradient} and Lemma \ref{lem:init_bound_disj}, $X^+_t$ can be strictly larger than $\max\left\{X^+_{t-1},\eta\abs{W_{t}^+(2)\cup W_{t}^+(4)}\right\}$ only if $\nett{t-1}(\xx^+) < \gamma$ and $-\nett{t-1}(\xx^-) \ge 1$. Furthermore, in this case $X^+_t - X^+_{t-1} = \eta\abs{W_t^+(2)\cup W_t^+(4)}$. By Lemma \ref{lem:dynamics}, $S_t^+$ increases in this case by $\eta\abs{W_t^+(1)\cup W_t^+(3)}$. We know by Lemma \ref{lem:stplus_lower_bound} that there exists $T_2 < T$ such that $S_{T_2}^+ \ge \gamma  + 1 -  3\ceta$ and that $\nett{t}(\xx^+) < \gamma$ and $-\nett{t}(\xx^-) \ge 1$ only for $t > T_2$. Since $S_t^+ \le \gamma + 1 + 8\ceta$ for all $t$ by Lemma \ref{lem:function_value_bound}, there can only be at most $\frac{11\ceta}{\eta\abs{W_T^+(1)\cup W_T^+(3)}}$ iterations in which $\nett{t}(\xx^+) < \gamma$ and $-\nett{t}(\xx^-) \ge 1$. It follows that 
	\begin{align*}
		X^+_t &\le \eta\abs{W_T^+(2)\cup W_T^+(4)} + \frac{11\ceta\eta\abs{W_T^+(2)\cup W_T^+(4)}}{\eta\abs{W_T^+(1)\cup W_T^+(3)}} \\ &\le \ceta + 11\ceta\frac{\left(\frac{k}{2}+2\sqrt{k}\right)}{\left(\frac{k}{2}-2\sqrt{k}\right)}
		\\ &\le 34\ceta
	\end{align*}
	where the second inequality follows by Lemma \ref{lem:init_num_disj} and the third inequality by the assumption on $k$.
	
	At convergence we have $\nett{T}(\xx^-) = S_T^- + X_T^+ - Y_T^+ - P_T^- \ge -1-3\ceta$ by Lemma \ref{lem:function_value_bound} (recall that $R_t^- = R_t^+ = X_t^+ - Y_t^+$). Furthermore, $P_T^- \ge 0$ and by Lemma \ref{lem:small_sums} we have $S_T^- \le \ceta$. Therefore, we get $Y_T^+ \le1 + 38\ceta$.
\end{proof}

We are now ready to prove the main result of this section.

\begin{prop}
	\label{prop:positive_gen}
	Define $\beta(\gamma) = \frac{\gamma -40\frac{1}{4}\ceta}{39\ceta + 1}$. Assume that $\gamma \ge 2$ and $k \ge 64\left(\frac{\beta(\gamma)+1}{\beta(\gamma) - 1}\right)^2$. Then with probability at least $1-\frac{\sqrt{2k}}{\sqrt{\pi} e^{8k}} - 8e^{-8}$, gradient descent converges to a global minimum which classifies all positive points correctly.
\end{prop}
\begin{proof}
	With probability at least $1-\frac{\sqrt{128k}}{\sqrt{\pi} e^{\frac{k}{2}}} - 8e^{-8}$ Proposition \ref{prop:opt}, and Lemma \ref{lem:init_num_disj_W1W3} hold. It suffices to show generalization on positive points. Assume that gradient descent converged to a global minimum at iteration $T$. Let $(\zz,1)$ be a positive point. Then there exists $\zz_i \in \{(1,1),(-1,-1)\}$. Assume without loss of generality that $\zz_i=(-1,-1) = \pp_3$.
	Define $$X^+_t(i) = \sum_{j \in W_T^+(i)}\left[\max\left\{\sigma\left(\wvec{j}\cdot \xx_1^+ \right),...,\sigma\left(\wvec{j}\cdot \xx_d^+ \right)\right\}\right]$$
	$$Y^+_t(i) = \sum_{j \in U_T^+(i)}\left[\max\left\{\sigma\left(\uvec{j}\cdot \xx_1^+ \right),...,\sigma\left(\uvec{j}\cdot \xx_d^+ \right)\right\}\right]$$
	for $i \in [4]$.
	
	Notice that 
	\begin{align*}
		\nett{T}(\xx^+) &=  X_T^+(1) + X_T^+(3) - P_T^+ + R_T^+  \\ &= X_T^+(1) + X_T^+(3) - P_T^+ + R_T^- \\ &= X_T^+(1) + X_T^+(3) - P_T^+ + \nett{T}(\xx^-) - S_T^- + P_T^-
	\end{align*}
	Since $\nett{T}(\xx^+) \ge \gamma$, $-\nett{T}(\xx^-) \ge 1$, $\abs{P_T^-} \le  \ceta$ by Lemma \ref{lem:small_sums} and  $P_T^+,S_T^- \ge 0$ , we obtain 
	\begin{equation}
		\label{eq:X1+X3}
		X_T^+(1) + X_T^+(3) \ge \gamma + 1 - \ceta
	\end{equation}
	Furthermore, by Lemma \ref{lem:proportion_13} we have 
	\begin{equation}
		\label{eq:X1X3}
		\frac{X_T^+(1) - X_0^+(1)}{\abs{W_T^+(1)}} = \frac{X_T^+(3) - X_0^+(3)}{\abs{W_T^+(3)}}
	\end{equation}	
	and by Lemma \ref{lem:init_num_disj_W1W3},
	\begin{equation}
		\label{eq:proportion_W1W3}
		\frac{\frac{k}{4}-2\sqrt{k}}{\frac{k}{4}+2\sqrt{k}} \le \frac{\abs{W_T^+(1)}}{\abs{W_T^+(3)}} \le \frac{\frac{k}{4}+2\sqrt{k}}{\frac{k}{4}-2\sqrt{k}}
	\end{equation}
	Let $\alpha(k) = \frac{\frac{k}{4}+2\sqrt{k}}{\frac{k}{4}-2\sqrt{k}}$. By Lemma \ref{lem:init_bound_disj} we have $\abs{X^+_0(1)} \le \frac{\eta k}{4} \le \frac{\ceta}{4}$.
	Combining this fact with \eqref{eq:X1X3} and \eqref{eq:proportion_W1W3} we get  $$X^+_T(1) \le \alpha(k)X^+_T(3) + X^+_0(1) \le \alpha(k)X^+_T(3)+\frac{\ceta}{4}$$ which implies together with \eqref{eq:X1+X3} that $X^+_T(3) \ge \frac{\gamma + 1- \frac{5\ceta}{4}}{1+\alpha(k)}$.
	Therefore, 
	\begin{align*}
		\nett{T}(\zz) &\ge X_T^+(3)- P_T^+ - Y_T^+(2) - Y_T^+(4) \\ &\ge \frac{\gamma + 1- \frac{5\ceta}{4}}{1+\alpha(k)} -\ceta -1 -3(8\ceta) -14\ceta \\ &= \frac{\gamma + 1- \frac{5\ceta}{4}}{1+\alpha(k)} -39\ceta - 1 > 0 \numberthis \label{eq:gen_positive}
	\end{align*}
	where the first inequality is true because
	\begin{align}
		\sum_{j = 1}^k\left[\max\left\{\sigma\left(\uvec{j}\cdot \zz_1 \right),...,\sigma\left(\uvec{j}\cdot \zz_d \right)\right\}\right] &\le \sum_{j = 1}^k\left[\max\left\{\sigma\left(\uvec{j}\cdot \xx_1^+ \right),...,\sigma\left(\uvec{j}\cdot \xx_d^+ \right)\right\}\right] \\ &= P_T^+ + Y_T^+(2) + Y_T^+(4) 
	\end{align}
	The second inequality in \eqref{eq:gen_positive} follows since $P_T^+ \le \ceta$ and by appyling Lemma \ref{lem:bounds24}. Finally, the last inequality in \eqref{eq:gen_positive} follows by the assumption on $k$. \footnote{The inequality $\frac{\gamma + 1- \frac{5\ceta}{4}}{1+\alpha(k)} -39\ceta - 1 > 0$ is equivalent to $\alpha(k) < \beta(\gamma)$ which is equivalent to $k > 64\left(\frac{\beta(\gamma)+1}{\beta(\gamma) - 1}\right)^2$.} Hence, $\zz$ is classified correctly.
\end{proof}

\subsubsection{Generalization on Negative Class}
\label{sec:negative_class}

We will need the following lemmas.
\begin{lem}
	\label{lem:init_num_disj_U}
	With probability at least $1-8e^{-8}$, it holds that $$\abs{\abs{U_0^+(2)} - \frac{k}{4}}\le 2\sqrt{k}$$ 
	$$\abs{\abs{U_0^+(4)} - \frac{k}{4}}\le 2\sqrt{k}$$
	$$\abs{\abs{ \left(U_0^+(1) \cup U_0^+(3)\right)\cap U_0^-(2)} - \frac{k}{4}}\le 2\sqrt{k} $$
	$$\abs{\abs{ \left(U_0^+(1) \cup U_0^+(3)\right)\cap U_0^-(4)} - \frac{k}{4}}\le 2\sqrt{k} $$
\end{lem}
\begin{proof}
	The proof is similar to the proof of Lemma \ref{lem:init_num_disj} and follows from the fact that 
	\begin{align*}
		\probarg{j \in U_0^+(2)} &= \probarg{j \in U_0^+(4)} \\ &= \probarg{j \in \left(U_0^+(1) \cup U_0^+(3)\right)\cap U_0^-(2)} \\ &=\probarg{j \in \left(U_0^+(1) \cup U_0^+(3)\right)\cap U_0^-(4)} \\ &=\frac{1}{4}
	\end{align*}
\end{proof}

\begin{lem}
	\label{lem:proportion_u24}
	Let $$X_t^- =  \sum_{j \in U_0^+(2)}\left[\max\left\{\sigma\left(\uvec{j}_t\cdot \xx_1^- \right),...,\sigma\left(\uvec{j}_t\cdot \xx_d^- \right)\right\}\right]$$
	and 
	$$Y_t^- =  \sum_{j \in U_0^+(4)}\left[\max\left\{\sigma\left(\uvec{j}_t\cdot \xx_1^- \right),...,\sigma\left(\uvec{j}_t\cdot \xx_d^- \right)\right\}\right] $$
	Then for all $t$, there exists $X,Y \ge 0$ such that $\abs{X} \le \eta\abs{U_{0}^+(2)}$, $\abs{Y} \le \eta\abs{U_{0}^+(4)}$ and $\frac{X_{t}^- - X}{\abs{U_{0}^+(2)}} = \frac{Y_{t}^- - Y}{\abs{U_{0}^+(4)}}$.
\end{lem}
\begin{proof}
	
	First, we will prove that for all $t$ there exists $a_t \in \mathbb{Z}$ such that for $j_1 \in U_0^-(2)$ and $j_2 \in U_0^-(4)$ it holds that $\uvec{j_1}_{t} = \uvec{j_1}_{0} + a_t\eta \pp_2$ and $\uvec{j_2}_{t} = \uvec{j_2}_{0} - a_t\eta \pp_2$. \footnote{Recall that by Lemma \ref{lem:u24} we know that $U_0^+(2) \cup U_0^+(4) \subseteq U_t^+(2) \cup U_t^+(4)$.} We will prove this by induction on $t$. 
	
	For $t=0$ this clearly holds. Assume it holds for an iteration $t$. Let $j_1 \in U_0^-(2)$ and $j_2 \in U_0^-(4)$. By the induction hypothesis, there exists $a_T \in \mathbb{Z}$ such that $\uvec{j_1}_{t} = \uvec{j_1}_{0} + a_t\eta \pp_2$ and $\uvec{j_2}_{t} = \uvec{j_2}_{0} - a_t\eta \pp_2$. Since for all $1 \le j \le k$ it holds that $\abs{\uvec{j}_0 \cdot \pp_2} < \frac{\sqrt{2}\eta}{4}$, it follows that either $U_0^-(2) \subseteq U_t^-(2)$ and $U_0^-(4) \subseteq U_t^-(4)$ or $U_0^-(2) \subseteq U_t^-(4)$ and $U_0^-(4) \subseteq U_t^-(2)$.
	In either case, by \eqref{eq:u_gradient}, we have the following update at iteration $t+1$:
	$$\uvec{j_1}_{t+1} = \uvec{j_1}_{t} + a\eta \pp_2 $$
	and 
	$$\uvec{j_2}_{t+1} = \uvec{j_2}_{t} - a\eta \pp_2 $$
	where $a \in \{-1,0,1\}$. Hence, $\uvec{j_1}_{t+1} = \uvec{j_1}_{0} + (a_t+a)\eta \pp_2$ and $\uvec{j_2}_{t} = \uvec{j_2}_{0} - (a_t+a)\eta \pp_2$. This concludes the proof by induction.
	
	Now, consider an iteration $t$, $j_1 \in U_0^+(2)$, $j_2 \in U_0^+(4)$ and the integer $a_t$ defined above. If $a_t \ge 0$ then 
	\begin{align*}
		\max\left\{\sigma\left(\uvec{j_1}_{t}\cdot \xx_1^- \right),...,\sigma\left(\uvec{j_1}_{t}\cdot \xx_d^- \right)\right\} - \max\left\{\sigma\left(\uvec{j_1}_{0}\cdot \xx_1^- \right),...,\sigma\left(\uvec{j_1}_{0}\cdot \xx_d^- \right)\right\} = \eta a_t	
	\end{align*}
	and 
	\begin{align*}
		\max\left\{\sigma\left(\uvec{j_2}_{t}\cdot \xx_1^- \right),...,\sigma\left(\uvec{j_2}_{t}\cdot \xx_d^- \right)\right\} - \max\left\{\sigma\left(\uvec{j_2}_{0}\cdot \xx_1^- \right),...,\sigma\left(\uvec{j_2}_{0}\cdot \xx_d^- \right)\right\} = \eta a_t
	\end{align*}
	Define $X = X_0^-$ and $Y = Y_0^-$ then $\abs{X} \le \eta\abs{U_{0}^-(2)}$, $\abs{Y} \le \eta\abs{U_{0}^-(4)}$ and
	$$\frac{X_{t}^- - X}{\abs{U_{0}^-(2)}} = \frac{\abs{U_{0}^-(2)}\eta a_t}{\abs{U_{0}^-(2)}} = \eta a_t = \frac{\abs{U_{0}^-(4)}\eta a_t}{\abs{U_{0}^-(4)}} = \frac{Y_{t}^- - Y}{\abs{U_{0}^-(4)}}$$
	which proves the claim in the case that $a_t \ge 0$.
	
	If $a_t < 0$ it holds that
	
	\begin{align*}
		\max\left\{\sigma\left(\uvec{j_1}_{t}\cdot \xx_1^- \right),...,\sigma\left(\uvec{j_1}_{t}\cdot \xx_d^- \right)\right\} - \max\left\{\sigma\left(\left(\uvec{j_1}_{0} - \pp_2\right)\cdot \xx_1^- \right),...,\sigma\left(\left(\uvec{j_1}_{0} - \pp_2\right)\cdot \xx_d^- \right)\right\} = \eta (-a_t-1)	
	\end{align*}
	and 
	\begin{align*}
		\max\left\{\sigma\left(\uvec{j_2}_{t}\cdot \xx_1^- \right),...,\sigma\left(\uvec{j_2}_{t}\cdot \xx_d^- \right)\right\} - \max\left\{\sigma\left(\left(\uvec{j_2}_{0} + \pp_2\right)\cdot \xx_1^- \right),...,\sigma\left(\left(\uvec{j_2}_{0} + \pp_2\right)\cdot \xx_d^- \right)\right\} = \eta (-a_t-1)	
	\end{align*}
	
	Define 
	
	$$X =  \sum_{j \in U_0^+(2)}\left[\max\left\{\sigma\left(\left(\uvec{j}_{0} - \pp_2\right)\cdot \xx_1^- \right),...,\sigma\left(\left(\uvec{j}_{0} - \pp_2\right)\cdot \xx_d^- \right)\right\}\right]$$
	and 
	$$Y =  \sum_{j \in U_0^+(4)}\left[\max\left\{\sigma\left(\left(\uvec{j}_{0} + \pp_2\right)\cdot \xx_1^- \right),...,\sigma\left(\left(\uvec{j}_{0} + \pp_2\right)\cdot \xx_d^- \right)\right\}\right]$$
	
	Since for all $1 \le j \le k$ it holds that $\abs{\uvec{j}_0 \cdot \pp_2} < \frac{\sqrt{2}\eta}{4}$, we have $\abs{X} \le \eta\abs{U_{0}^-(2)}$, $\abs{Y} \le \eta\abs{U_{0}^-(4)}$. Furthermore, 
	$$\frac{X_{t}^- - X}{\abs{U_{0}^-(2)}} = \frac{\abs{U_{0}^-(2)}\eta (-a_t-1)}{\abs{U_{0}^-(2)}} = \eta (-a_t-1) = \frac{\abs{U_{0}^-(4)}\eta (-a_t-1)}{\abs{U_{0}^-(4)}} = \frac{Y_{t}^- - Y}{\abs{U_{0}^-(4)}}$$
	which concludes the proof.
\end{proof}

\begin{lem}
	\label{lem:proportion_u13}
	Let $$X_t^- =  \sum_{j \in \left(U_0^+(1) \cup U_0^+(3)\right)\cap U_0^-(2)}\left[\max\left\{\sigma\left(\uvec{j}_t\cdot \xx_1^- \right),...,\sigma\left(\uvec{j}_t\cdot \xx_d^- \right)\right\}\right]$$
	and 
	$$Y_t^- =  \sum_{j \in \left(U_0^+(1) \cup U_0^+(3)\right)\cap U_0^-(4)}\left[\max\left\{\sigma\left(\uvec{j}_t\cdot \xx_1^- \right),...,\sigma\left(\uvec{j}_t\cdot \xx_d^- \right)\right\}\right] $$
	Then for all $t$, $\frac{X_{t}^- - X_{0}^-}{\abs{\left(U_0^+(1) \cup U_0^+(3)\right)\cap U_0^-(2)}} = \frac{Y_{t}^- - Y_{t}^-}{\abs{\left(U_0^+(1) \cup U_0^+(3)\right)\cap U_0^-(4)}}$.
\end{lem}
\begin{proof}
	
	We will first prove that for all $t$ there exists an integer $a_t \ge 0$ such that for $j_1 \in \left(U_0^+(1) \cup U_0^+(3)\right)\cap U_0^-(2)$ and $j_2 \in \left(U_0^+(1) \cup U_0^+(3)\right)\cap U_0^-(4)$ it holds that $\uvec{j_1}_{t}\cdot \pp_2 = \uvec{j_1}_{0} \cdot \pp_2+ \eta a_t$ and $\uvec{j_2}_{t}\cdot \pp_4 = \uvec{j_2}_{0}\cdot \pp_4 + \eta a_t$. We will prove this by induction on $t$.

	For $t=0$ this clearly holds. Assume it holds for an iteration $t$. Let $j_1 \in \left(U_0^+(1) \cup U_0^+(3)\right)\cap U_0^-(2)$ and $j_2 \in \left(U_0^+(1) \cup U_0^+(3)\right)\cap U_0^-(4)$. By the induction hypothesis, there exists an integer $a_t \ge 0$ such that $\uvec{j_1}_{t}\cdot \pp_2 = \uvec{j_1}_{0} \cdot \pp_2+ \eta a_t$ and $\uvec{j_2}_{t}\cdot \pp_4 = \uvec{j_2}_{0}\cdot \pp_4 + \eta a_t$. Since for all $1 \le j \le k$ it holds that $\abs{\uvec{j}_0 \cdot \pp_1} < \frac{\sqrt{2}\eta}{4}$, it follows that if $a_t \ge 1$ we have the following update at iteration $T+1$:
	$$\uvec{j_1}_{t+1} = \uvec{j_1}_{t} + a\eta \pp_2 $$
	and 
	$$\uvec{j_2}_{t+1} = \uvec{j_2}_{t} + a\eta \pp_4 $$
	where $a \in \{-1,0,1\}$. Hence, $\uvec{j_1}_{t+1}\cdot \pp_2 = \uvec{j_1}_{0} \cdot \pp_2+ \eta (a_t+a)$ and $\uvec{j_2}_{t+1}\cdot \pp_4 = \uvec{j_2}_{0}\cdot \pp_4 + \eta (a_t+a)$.
	
	Otherwise, if $a_t = 0$ then 
	$$\uvec{j_1}_{t+1} = \uvec{j_1}_{t} + a\eta \pp_2 + b_1\pp_1$$
	and 
	$$\uvec{j_2}_{t+1} = \uvec{j_2}_{t} + a\eta \pp_4 + b_2\pp_1$$
	such that $a \in \{0,1\}$ and $b_1,b_2 \in \{-1,0,1\}$. Hence, $\uvec{j_1}_{t+1}\cdot \pp_2 = \uvec{j_1}_{0} \cdot \pp_2+ \eta (a_t+a)$ and $\uvec{j_2}_{t+1}\cdot \pp_4 = \uvec{j_2}_{0}\cdot \pp_4 + \eta (a_t+a)$. This concludes the proof by induction.
	
	Now, consider an iteration $t$, $j_1 \in \left(U_0^+(1) \cup U_0^+(3)\right)\cap U_0^-(2)$ and $j_2 \in \left(U_0^+(1) \cup U_0^+(3)\right)\cap U_0^-(4)$ and the integer $a_t$ defined above. We have,
	
	\begin{align*}
		\max\left\{\sigma\left(\uvec{j_1}_{t}\cdot \xx_1^- \right),...,\sigma\left(\uvec{j_1}_{t}\cdot \xx_d^- \right)\right\} - \max\left\{\sigma\left(\uvec{j_1}_{0}\cdot \xx_1^- \right),...,\sigma\left(\uvec{j_1}_{0}\cdot \xx_d^- \right)\right\} = \eta a_t	
	\end{align*}
	and 
	\begin{align*}
		\max\left\{\sigma\left(\uvec{j_2}_{t}\cdot \xx_1^- \right),...,\sigma\left(\uvec{j_2}_{t}\cdot \xx_d^- \right)\right\} - \max\left\{\sigma\left(\uvec{j_2}_{0}\cdot \xx_1^- \right),...,\sigma\left(\uvec{j_2}_{0}\cdot \xx_d^- \right)\right\} = \eta a_t
	\end{align*}
	It follows that
	\begin{align*}
		\frac{X_{t}^- - X_{0}^-}{\abs{\left(U_0^+(1) \cup U_0^+(3)\right)\cap U_0^-(2)}} &= \frac{\abs{\left(U_0^+(1) \cup U_0^+(3)\right)\cap U_0^-(2)}\eta a_t}{\abs{\left(U_0^+(1) \cup U_0^+(3)\right)\cap U_0^-(2)}} \\ &= \eta a_t \\ &= \frac{\abs{\left(U_0^+(1) \cup U_0^+(3)\right)\cap U_0^-(4)}\eta a_t}{\abs{\left(U_0^+(1) \cup U_0^+(3)\right)\cap U_0^-(4)}} \\ &= \frac{Y_{t}^- - Y_{0}^-}{\abs{\left(U_0^+(1) \cup U_0^+(3)\right)\cap U_0^-(4)}}	
	\end{align*}
	$$$$
	which concludes the proof.
\end{proof}

We are now ready to prove the main result of this section.

\begin{prop}
	\label{prop:negative_gen}
	Define $\beta = \frac{1- 36\frac{1}{4}\ceta}{35\ceta}$. Assume that $k > 64\left(\frac{\beta+1}{\beta - 1}\right)^2$. Then with probability at least $1-\frac{\sqrt{2k}}{\sqrt{\pi} e^{8k}} - 8e^{-8}$, gradient descent converges to a global minimum which classifies all negative points correctly.
\end{prop}
\begin{proof}
	With probability at least $1-\frac{\sqrt{2k}}{\sqrt{\pi} e^{8k}} - 16e^{-8}$ Proposition \ref{prop:opt} and Lemma \ref{lem:init_num_disj_U} hold. It suffices to show generalization on negative points. Assume that gradient descent converged to a global minimum at iteration $T$. Let $(\zz,-1)$ be a negative point. Assume without loss of generality that $\zz_i=\pp_2$ for all $1 \le i \le d$.
	Define the following sums for $l\in\{2,4\}$, $$X^-_t = \sum_{j \in W_t^+(2)\cup W_t^+(4)}\left[\max\left\{\sigma\left(\wvec{j}\cdot \xx_1^- \right),...,\sigma\left(\wvec{j}\cdot \xx_d^- \right)\right\}\right]$$
	$$Y_t^-(l) =  \sum_{j \in U_0^+(l)}\left[\max\left\{\sigma\left(\uvec{j}_t\cdot \xx_1^- \right),...,\sigma\left(\uvec{j}_t\cdot \xx_d^- \right)\right\}\right] $$
	$$Z^-_t(l) = \sum_{j \in \left(U_0^+(1) \cup U_0^+(3)\right)\cap U_0^-(l)}\left[\max\left\{\sigma\left(\uvec{j}\cdot \xx_1^- \right),...,\sigma\left(\uvec{j}\cdot \xx_d^- \right)\right\}\right]$$
	
	First, we notice that
	\begin{align*}
		\nett{T}(\xx^-) &=  S_T^- + X_T^- - Y_T^-(2) - Y_T^-(4) - Z^-_T(2) - Z^-_T(4)
	\end{align*}
	$$X_T^-, S_T^- \ge 0 $$
	and $$ \nett{T}(\xx^-) \le -1$$ 
	imply that
	\begin{equation}
		\label{eq:Y+Z_lb}
		Y_T^-(2) + Y_T^-(4) + Z^-_T(2) + Z^-_T(4) \ge 1
	\end{equation}
	
	We note that by the analysis in Lemma \ref{lem:init_num_disj_U}, it holds that for any $t$, $j_1 \in U_0^+(2)$ and $j_2 \in U_0^+(4)$, either $j_1 \in U_t^+(2)$ and $j_2 \in U_t^+(4)$, or $j_1 \in U_t^+(4)$ and $j_2 \in U_t^+(2)$. We assume without loss of generality that $j_1 \in U_T^+(2)$ and $j_2 \in U_T^+(4)$. It follows that in this case $\nett{T}(\zz) \le S_T^- + X_T^- - Z^-_T(2) - Y^-_T(2)$. \footnote{The fact that we can omit the term $-Z^-_T(4)$ from the latter inequality follows from Lemma \ref{lem:u13}.}Otherwise we would replace $Y^-_T(2)$ with $Y^-_T(4)$ and vice versa and continue with the same proof.
	
	Let $\alpha(k) = \frac{\frac{k}{4}+2\sqrt{k}}{\frac{k}{4}-2\sqrt{k}}$. By Lemma \ref{lem:proportion_u13} and Lemma \ref{lem:init_num_disj_U}
	$$Z^-_T(4) \le \alpha(k)Z^-_T(2) + Z^-_0(2) \le \alpha(k)Z^-_T(2) + \frac{\ceta}{4}$$
	and by Lemma \ref{lem:proportion_u24} and Lemma \ref{lem:init_num_disj_U} there exists $Y \le \ceta$ such that:
	$$Y^-_T(4) \le \alpha(k)Y^-_T(2) + Y \le \alpha(k)Y^-_T(2) + \ceta$$
	
	Plugging these inequalities in \eqref{eq:Y+Z_lb} we get:
	
	$$\alpha(k)Z^-_T(2) + \frac{\ceta}{4} + \alpha(k)Y^-_T(2) + \ceta +Y_T^-(2) + Z_T^-(2) \ge 1$$
	which implies that
	$$Y_T^-(2) + Z_T^-(2) \ge \frac{1 - \frac{5\ceta}{4}}{\alpha(k)+1} $$
	By Lemma \ref{lem:bounds24} we have $X_T^- \le 34\ceta$. Hence, by using the inequality $S_T^- \le \ceta$ we conclude that
	\begin{align*}
		\nett{T}(\zz) \le S_T^- + X_T^- - Z^-_T(2) - Y^-_T(2) \le 35\ceta - \frac{1 - \frac{5\ceta}{4}}{\alpha(k)+1} < 0
	\end{align*}
	
	where the last inequality holds for $k > 64\left(\frac{\beta+1}{\beta - 1}\right)^2$. \footnote{It holds that $35\ceta - \frac{1 - \frac{5\ceta}{4}}{\alpha(k)+1} < 0$ if and only if $\alpha(k) < \beta$ which holds if and only if $k > 64\left(\frac{\beta+1}{\beta - 1}\right)^2$.} Therefore, $\zz$ is classified correctly.
\end{proof}

\subsubsection{Finishing the Proof}
\label{sec:finishing_proof}

First, for $k \ge 120$, with probability at least $1-\frac{\sqrt{2k}}{\sqrt{\pi} e^{8k}} - 16e^{-8}$, Proposition \ref{prop:opt}, Lemma \ref{lem:init_num_disj_W1W3} and  Lemma \ref{lem:init_num_disj_U} hold. Also, for the bound on $T$, note that in this case $\frac{28(\gamma + 1 + 8\ceta)}{\ceta} \ge \frac{7(\gamma + 1 + 8\ceta)}{\left(\frac{k}{2} - 2\sqrt{k}\right)\eta}$. Define $\beta_1 = \frac{\gamma -40\frac{1}{4}\ceta}{39\ceta + 1}$ and $\beta_2 = \frac{1- 36\frac{1}{4}\ceta}{35\ceta}$ and let $\beta = \max\{\beta_1,\beta_2\}$. For $\gamma \ge 8$ and $\ceta \le \frac{1}{410}$ it holds that $64\left(\frac{\beta+1}{\beta - 1}\right)^2 < 120$. By Proposition \ref{prop:positive_gen} and Proposition \ref{prop:negative_gen}, it follows that for $k \ge 120$ gradient descent converges to a global minimum which classifies all points correctly.

We will now prove the clustering effect at a global minimum.  By Lemma \ref{lem:stplus_lower_bound} it holds that $S_{T}^+ \ge \gamma + 1 - 3 \ceta \ge \gamma - 1$. Therefore, by Lemma \ref{lem:same_w} it follows that $$2\eta(a^+(T) + 1) \abs{W^+_0(1) \cup W^+_0(3)} \ge S_T^+ \ge \gamma - 1$$
and thus $a^+(T) \ge \frac{\gamma - 1}{2\ceta} - 1$. Therefore, for any $j \in W^+_0(i)$ such that $i\in \{1,3\}$, the cosine of the angle between $\wvec{j}_T$ and $\pp_i$ is at least
$$\frac{ (\wvec{j}_0 + a^+(T)\eta\pp_1 + \alpha_i^t\pp_2) \cdot \pp_1}{\sqrt{2}(\norm{\wvec{j}_0} + \sqrt{2}a^+(T)\eta + \sqrt{2}\eta)} \ge \frac{2a^+(T)}{2a_1(T) + 3} \ge \frac{\gamma - 1-2\ceta}{\gamma-1+\ceta}$$

where we used the triangle inequality and Lemma  \ref{lem:same_w}. The claim follows.


\section{Proof of Theorem \ref{thm:lower_bound_specific}}
\label{sec:proof_lower_bound}

\begin{thm}
	\label{thm:lower_bound_specific_restated}
	(\textbf{Theorem \ref{thm:lower_bound_specific} restated}) Assume that gradient descent runs with parameaters $\eta = \frac{\ceta}{k}$ where $\ceta \le \frac{1}{41}$, $\sigma_g \le \frac{\ceta}{16k^{\frac{3}{2}}}$ and $\gamma \ge 1$. Then, with probability at least $\left(1-c\right)\frac{33}{48}$, gradient descent converges to a global minimum that does not recover $f^*$. Furthermore, there exists $1 \le i \le 4$ such that the global minimum misclassifies all points $\xx$ such that $P_{\xx} = A_i$.
\end{thm}

We refer to \eqref{eq:opt_increase} in the proof of Proposition \ref{prop:opt}. To show convergence and provide convergence rates of gradient descent, the proof uses Lemma \ref{lem:init_num_disj}. However, to only show convergence, it suffices to bound the probability that $W_0^+(1) \cup W_0^+(3) \neq \emptyset$ and that the initialization satisfies Lemma \ref{lem:init_bound_disj}. Given that Lemma \ref{lem:init_bound_disj} holds (with probability at least $1-\sqrt{\frac{8}{\pi}}e^{-32}$), then $W_0^+(1) \cup W_0^+(3) \neq \emptyset$ holds with probability $\frac{3}{4}$.

By the argument above, with probability at least $\left(1-\sqrt{\frac{8}{\pi}}e^{-32}\right)\frac{3}{4}$, Lemma \ref{lem:init_bound_disj} holds with $k=2$ and $W_0^+(1) \cup W_0^+(3) \neq \emptyset$ which implies that gradient descent converges to a global minimum. For the rest of the proof we will condition on the corresponding event. Let $T$ be the iteration in which gradient descent converges to a global minimum. Note that $T$ is a random variable. Denote the network at iteration $T$ by $N$. For all $\zz \in \reals^{2d}$ denote  $$N(\zz) = \sum_{j=1}^{2}\left[\max\left\{\sigma\left(\wvec{j}\cdot \zz_1 \right),...,\sigma\left(\wvec{j}\cdot \zz_d \right)\right\} - \max\left\{\sigma\left(\uvec{j}\cdot \zz_1 \right),...,\sigma\left(\uvec{j}\cdot \zz_d \right)\right\}\right]$$

Let $E$ denote the event for which at least one of the following holds:
\begin{enumerate}
	\label{en:event_E}
	\item $W_T^+(1) = \emptyset$.
	\item $W_T^+(3) = \emptyset$.
	\item $\uvec{1} \cdot \pp_2 > 0$ and $\uvec{2} \cdot \pp_2 > 0$.
	\item $\uvec{1} \cdot \pp_4 > 0$ and $\uvec{2} \cdot \pp_4 > 0$.
\end{enumerate}

Our proof will proceed as follows. We will first show that if $E$ occurs then gradient descent does not learn $f^*$, i.e., the network $N$ does not satisfy $\sign\left(N(\xx)\right) = f^*(\xx)$ for all $\xx \in \{\pm1\}^{2d}$.
Then, we will show that $\probarg{E} \ge \frac{11}{12}$. This will conclude the proof.

Assume that one of the first two items in the definition of the event $E$ occurs. Without loss of generality assume that $W_T^+(1) = \emptyset$ and recall that $\xx^-$ denotes a negative vector which only contains the patterns $\pp_2, \pp_4$ and let $\zz^+ \in \reals^{2d}$ be a positive vector which only contains the patterns $\pp_1, \pp_2, \pp_4$. By the assumption $W_T^+(1) = \emptyset$ and the fact that $\pp_1 = -\pp_3$ it follows that for all $j = 1,2$, $$\max\left\{\sigma\left(\wvec{j}\cdot \zz^+_1 \right),...,\sigma\left(\wvec{j}\cdot \zz^+_d \right)\right\} = \max\left\{\sigma\left(\wvec{j}\cdot \xx^-_1 \right),...,\sigma\left(\wvec{j}\cdot \xx^-_d \right)\right\} $$

Furthermore, since $\zz^+$ contains more distinct patterns than $\xx^-$, it follows that for all $j = 1,2$,
$$\max\left\{\sigma\left(\uvec{j}\cdot \zz^+_1 \right),...,\sigma\left(\uvec{j}\cdot \zz^+_d \right)\right\} \ge \max\left\{\sigma\left(\uvec{j}\cdot \xx^-_1 \right),...,\sigma\left(\uvec{j}\cdot \xx^-_d \right)\right\} $$

Hence, $N(\zz^+) \le N(\xx^-)$. Since at a global minimum $N(\xx^-) \le -1$, we have $N(\zz^+) \le -1$ and $\zz_2$ is not classified correctly.

Now assume without loss of generality that the third item in the definition of $E$ occurs. Let $\zz^-$ be the negative vector with all of its patterns equal to $\pp_4$. It is clear that $N(\zz^-) \ge 0$ and therefore $\zz^-$ is not classified correctly. This concludes the first part of the proof. We will now proceed to show that $\probarg{E} \ge \frac{11}{12}$.

Denote by $A_i$ the event that item $i$ in the definition of $E$ occurs and for an event $A$ denote by $A^c$ its complement. Thus $E^c = \cap_{i=1}^4{A_i^c}$ and $\probarg{E^c} = \probarg{A_3^c\cap A_4^c \mid A_1^c\cap A_2^c} \probarg{A_1^c\cap A_2^c}$.

We will first calculate $\probarg{A_1^c\cap A_2^c}$. By Lemma \ref{lem:same_w}, we know that for $i \in \{1,3\}$, $W_0^+(i) = W_T^+(i)$. Therefore, it suffices to calculate the probabilty that $W_0^+(1) \neq \emptyset$ and $W_0^+(3) \neq \emptyset$, provided that $W_0^+(1) \cup W_0^+(3) \neq \emptyset$. Without conditioning on $W_0^+(1) \cup W_0^+(3) \neq \emptyset$, for each $1 \le i \le 4$ and $1 \le j \le 2$ the event that $j \in W_0^+(i)$ holds with probability $\frac{1}{4}$. Since the initializations of the filters are independent, we have $\probarg{A_1^c\cap A_2^c} = \frac{1}{6}$.  \footnote{Note that this holds after conditioning on the corresponding event of Lemma \ref{lem:init_bound_disj}.} 

We will show that $\probarg{A_3^c\cap A_4^c \mid A_1^c\cap A_2^c} = \frac{1}{2}$ by a symmetry argument. This will finish the proof of the theorem. For the proof, it will be more convenient to denote the matrix of weights at iteration $t$ as a tuple of 4 vectors, i.e.,  $W_t=\left(\wvec{1}_0,\wvec{2}_0,\uvec{1}_0,\uvec{2}_0\right)$. Consider two initializations $W_0^{(1)}=\left(\wvec{1}_0,\wvec{2}_0,\uvec{1}_0,\uvec{2}_0\right)$ and $W_0^{(2)}=\left(\wvec{1}_0,\wvec{2}_0,-\uvec{1}_0,\uvec{2}_0\right)$ and let $W_t^{(1)}$ and $W_t^{(2)}$ be the corresponding weight values at iteration $t$. We will prove the following lemma:
\begin{lem}
	\label{lem:symmetry}
	For all $t \ge 0$, if $W_t^{(1)}=\left(\wvec{1}_t,\wvec{2}_t,\uvec{1}_t,\uvec{2}_t\right)$ then $W_t^{(2)}=\left(\wvec{1}_t,\wvec{2}_t,-\uvec{1}_t,\uvec{2}_t\right)$.
\end{lem}
\begin{proof}
	We will show this by induction on $t$. \footnote{Recall that we condition on the event corresponding to Lemma \ref{lem:init_bound_disj}. By negating a weight vector we still satisfy the bounds in the lemma and therefore the claim that will follow will hold under this conditioning.}This holds by definition for $t=0$. Assume it holds for an iteration $t$. Denote $W_{t+1}^{(2)}=\left(\zz_1,\zz_2,\vv_1,\vv_2\right)$. We need to show that $\zz_1 = \wvec{1}_{t+1}$, $\zz_2 = \wvec{2}_{t+1}$, $\vv_1 = -\uvec{1}_{t+1}$ and $\vv_2 = \uvec{2}_{t+1}$. By the induction hypothesis it holds that $N_{W_t^{(1)}}(\xx^+) = N_{W_t^{(2)}}(\xx^+)$ and $N_{W_t^{(1)}}(\xx^-) = N_{W_t^{(2)}}(\xx^-)$. This follows since for diverse points (either positive or negative), negating a neuron does not change the function value. Thus, according to \eqref{eq:w_gradient} and \eqref{eq:u_gradient} we have $\zz_1 = \wvec{1}_{t+1}$, $\zz_2 = \wvec{2}_{t+1}$ and $\vv_2 = \uvec{2}_{t+1}$. We are left to show that $\vv_1 = -\uvec{1}_{t+1}$. This follows from \eqref{eq:u_gradient} and the following  facts: 
	\begin{enumerate}
		\item $\pp_3=-\pp_1$.
		\item $\pp_2=-\pp_4$.
		\item $\argmax_{1 \le l \le 4} \uu\cdot \pp_l = 1$ if and only if $\argmax_{1 \le l \le 4} -\uu\cdot \pp_l = 3$.
		\item $\argmax_{1 \le l \le 4} \uu\cdot \pp_l = 2$ if and only if $\argmax_{1 \le l \le 4} -\uu\cdot \pp_l = 4$.
		\item $\argmax_{l \in \{2,4\}} \uu\cdot \pp_l = 2$ if and only if $\argmax_{l \in \{2,4\}} -\uu\cdot \pp_l = 4$.
	\end{enumerate}
	To see this, we will illustrate this through one case, the other cases are similar. Assume, for example, that $\argmax_{1 \le l \le 4} \uvec{1}_t\cdot \pp_l = 3$ and $\argmax_{l \in \{2,4\}} \uvec{1}_t\cdot \pp_l = 2$ and assume without loss of generality that $N_{W_t^{(1)}}(\xx^+) = N_{W_t^{(2)}}(\xx^+) < \gamma$ and $N_{W_t^{(1)}}(\xx^-) = N_{W_t^{(2)}}(\xx^-) > -1$. Then, by \eqref{eq:u_gradient}, $\uvec{1}_{t+1} = \uvec{1}_t - \pp_3 + \pp_2$. By the induction hypothesis and the above facts it follows that $\vv_1 = -\uvec{1}_t - \pp_1 + \pp_4 = -\uvec{1}_t + \pp_3 - \pp_2 =  -\uvec{1}_{t+1}$. This concludes the proof.
\end{proof}

Consider an initialization of gradient descent where $\wvec{1}_0$ and $\wvec{2}_0$ are fixed and the event that we conditioned on in the beginning of the proof and $A_1^c\cap A_2^c$ hold. Define the set $B_1$ to be the set of all pair of vectors $(\vv_1, \vv_2)$ such that if $\uvec{1}_0 = \vv_1$ and $\uvec{1}_0 = \vv_2$ then at iteration $T$, $\uvec{1} \cdot \pp_2 > 0$ and $\uvec{2} \cdot \pp_2 > 0$. Note that this definition implicitly implies that this initialization satisfies the condition in Lemma \ref{lem:init_bound_disj} and leads to a global minimum. Similarly, let $B_2$ be the set of all pair of vectors $(\vv_1, \vv_2)$ such that if $\uvec{1}_0 = \vv_1$ and $\uvec{1}_0 = \vv_2$ then at iteration $T$, $\uvec{1} \cdot \pp_4 > 0$ and $\uvec{2} \cdot \pp_2 > 0$.
First, if $(\vv_1,\vv_2) \in B_1$ then $(-\vv_1,\vv_2)$ satisfies the conditions of Lemma \ref{lem:init_bound_disj}.
Second, by Lemma \ref{lem:symmetry}, it follows that if $(\vv_1,\vv_2) \in B_1$ then initializating with $(-\vv_1,\vv_2)$, leads to the same values of $\nett{t}(\xx^+)$ and $\nett{t}(\xx^-)$ in all iterations $0 \le t \le T$. Therefore, initializing with $(-\vv_1,\vv_2)$ leads to a convergence to a global minimum with the same value of $T$ as the initialization with $(\vv_1,\vv_2)$. Furthermore, if $(\vv_1,\vv_2) \in B_1$, then by Lemma \ref{lem:symmetry}, initializing with $\uvec{1}_0 = -\vv_1$ and $\uvec{1}_0 = \vv_2$ results in $\uvec{1} \cdot \pp_2 < 0$ and $\uvec{2} \cdot \pp_2 > 0$. It follows that $(\vv_1,\vv_2) \in B_1$ if and only if $(-\vv_1,\vv_2) \in B_2$. 

For $l_1,l_2 \in \{2,4\}$ define $P_{l_1,l_2} = \probarg{\uvec{1} \cdot \pp_{l_1} > 0 \wedge \uvec{2} \cdot \pp_{l_2} > 0 \mid A_1^c\cap A_2^c, \wvec{1}_0, \wvec{2}_0}$
Then, by symmetry of the initialization and the latter arguments it follows that $P_{2,2} = P_{4,2}$.

By similar arguments we can obtain the  equalities $P_{2,2} = P_{4,2} = P_{4,4} = P_{2,4}$.

Since all of these four probabilities sum to $1$, each is equal to $\frac{1}{4}$. \footnote{Note that the probablity that $\uvec{i} \cdot \pp_j = 0$ is 0 for all possible $i$ and $j$.}Taking expectations of these probabilities with respect to the values of $\wvec{1}_0$ and $\wvec{2}_0$ (given that Lemma \ref{lem:init_bound_disj} and $A_1^c\cap A_2^c$ hold) and using the law of total expectation, we conclude that 
\begin{align*}
	\probarg{A_3^c\cap A_4^c \mid A_1^c\cap A_2^c} &=  \probarg{\uvec{1} \cdot \pp_4 > 0 \wedge \uvec{2} \cdot \pp_2 > 0 \mid A_1^c\cap A_2^c} \\ &+ \probarg{\uvec{1} \cdot \pp_2 > 0 \wedge \uvec{2} \cdot \pp_4 > 0 \mid A_1^c\cap A_2^c} =  \frac{1}{2}
\end{align*}

Finally, let $\mz_1$ be the set of positive points which contain only the patterns $\pp_1$, $\pp_2$, $\pp_4$, $\mz_2$ be the set of positive points which contain only the patterns $\pp_3$, $\pp_2$, $\pp_4$. Let $\mz_3$ be the set which contains the negative point with all patterns equal to $\pp_2$ and $\mz_4$ be the set which contains the negative point with all patterns equal to $\pp_4$. By the proof of the previous section, if the event $E$ holds, then there exists $1 \le i \le 4$, such that gradient descent converges to a solution at iteration $T$ which errs on all of the points in $\mz_i$.  
Therefore, its test error will be at least $p^*$ (recall \eqref{eq:pstar}).
\ignore{
	Let $\md$ be a probability distribution where the probability for a positive point is $0.5$. Furthermore, the probability for each $\zz_1$ and $\zz_2$ is $\frac{1-p_+}{4}$ and the probability for each $\zz_3$ and $\zz_4$ is $\frac{1-p_-}{4}$. In this case, if the event $E$ holds, then gradient descent will have test error at least $\min\left\{\frac{1-p_+}{4}, \frac{1-p_{-}}{4}\right\}$, which concludes the proof.
}

\ignore{
	\subsection{Proof of Theorem \ref{thm:lower_bound_general}}
	\label{sec:proof_lower_bound_general}
	
	The theorem follows by a simple observation. By the gradient update in \eqref{eq:w_gradient} it follows that for $i \in \{1,3\}$ if $\wvec{1}_0 \cdot \xx_i < 0$ and $\wvec{2}_0 \cdot \xx_i < 0$, then for any iteration $t$, $\wvec{1}_t \cdot \xx_i < 0$ and $\wvec{2}_t \cdot \xx_i < 0$. Furthermore, if $\wvec{1}_t \cdot \xx_i < 0$ and $\wvec{2}_t \cdot \xx_i < 0$ at convergence, then for the positive point $\zz_i \in \reals^{2d}$, which all of its patterns are equal to $\xx_i$, it holds that $\nett{t}(\zz_i) \le 0$. Therefore, $\zz_i$ is not classified correctly. Finally, the probability that $\wvec{1}_0 \cdot \xx_1 < 0$ and $\wvec{2}_0 \cdot \xx_1 < 0$ or $\wvec{1}_0 \cdot \xx_3 < 0$ and $\wvec{2}_0 \cdot \xx_3 < 0$ is equal to $\frac{1}{2}$.
	
}

\section{Proof of Theorem \ref{thm:generalization_gap}}
\label{sec:cor_proof}

Let $\delta \ge 1-p_+p_-(1-c-16e^{-8})$. By Theorem \ref{thm:main}, given 2 samples, one positive and one negative, with probability at least $1-\delta \le p_+p_-(1-c-16e^{-8})$, gradient descent will converge to a global minimum that has 0 test error. Therefore, for all $\epsilon \ge 0$, $m(\epsilon, \delta) \le 2$. On the other hand, by Theorem \ref{thm:lower_bound_specific}, if $m < \frac{2\log\left(\frac{48\delta}{33(1-c)}\right)}{\log(p_+p_-)}$ then with probability greater than $$\left(p_+p_-\right)^{\frac{\log\left(\frac{48\delta}{33(1-c)}\right)}{\log(p_+p_-)}}(1-c)\frac{33}{48} = \delta$$
gradient descent converges to a global minimum with test error at least $p^*$. It follows that for $0 \le \epsilon < p^*$, $m(\epsilon, \delta) \ge \frac{2\log\left(\frac{48\delta}{33(1-c)}\right)}{\log(p_+p_-)}$.

\section{Experiments for Section \ref{sec:experiments}}

We first provide several details on the experiments in Section \ref{sec:experiments}. We trained the overparamaterized network with 120 channels once for each training set size and recorded the clustered weights. We used Adam for optimization and batch size which is one-tenth of the size of the training set. We used learning rate=0.01 and standard deviation of 0.05 for initialization with truncated normal weights. For the small network with random initialization we used the same optimization method and batch sizes but tried 6 different pairs of values for learning rate and standard deviation: (0.01,0.01), (0.01,0.05), (0.05,0.05), (0.05, 0.01), (0.1,0.5) and (0.1,0.1). For each pair and training set size we trained 20 times and averaged the results. The curve is the best test accuracy we got among all learning rate and standard deviation pairs.

For the small network with cluster initialization we experimented with the same setup as the small network with random initializatoin but only experimented with learning rate 0.01 and standard deviation 0.05. The curve is an average of 20 runs for each training set size.

We also experimented with other filter sizes in similar setups. Figure \ref{fig:exp_mnist4x4} shows the results for 4x4 filters and clustering from 120 filters to 4 filters (with 2000 training points). Figure \ref{fig:exp_mnist7x7} shows the results for 7x7 filters and clustering from 120 filters to 4 filters (with 2000 training points).

\begin{figure*}[t]
	\begin{subfigure}{.5\textwidth}
		\centering
		\includegraphics[width=0.8\linewidth]{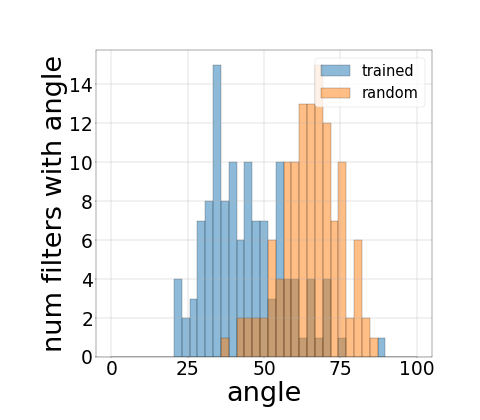}
		\caption{}
		\label{fig:exp_mnist4x4_1}
	\end{subfigure}%
	\begin{subfigure}{.5\textwidth}
		\centering
		\includegraphics[width=0.8\linewidth]{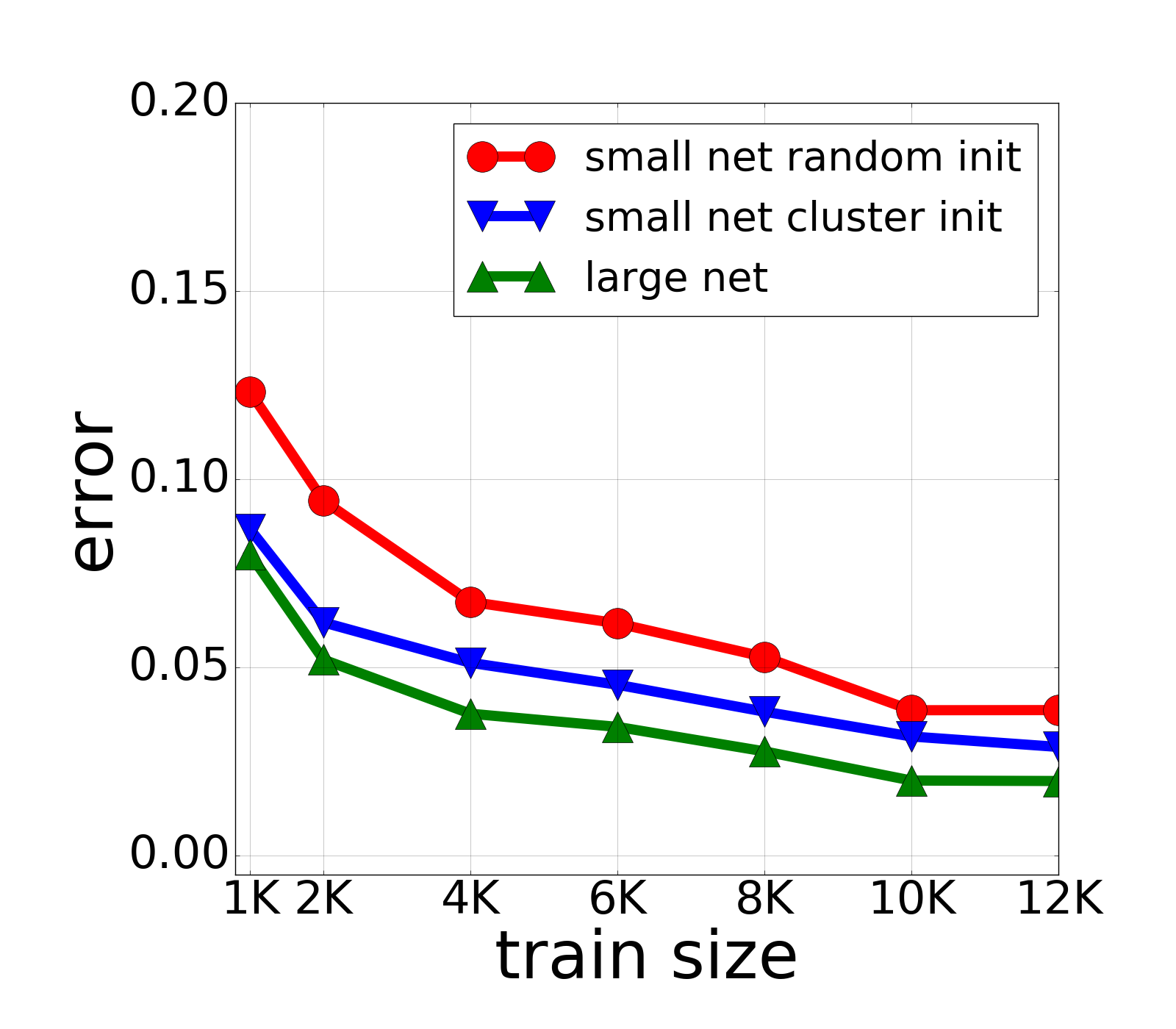}
		\caption{}
		\label{fig:exp_mnist4x4_2}
	\end{subfigure}%
	\caption{\small{ Clustering and Exploration in MNIST with 4x4 filters (a) Distribution of angle to closest center in trained and random networks. (b) The plot shows the test error of the small network (4 channels) with standard training (red), the small network that uses clusters from the large network (blue), and the large network (120 channels) with standard training (green).}}
	\label{fig:exp_mnist4x4}
\end{figure*}

\begin{figure*}[t]
	\begin{subfigure}{.5\textwidth}
		\centering
		\includegraphics[width=0.8\linewidth]{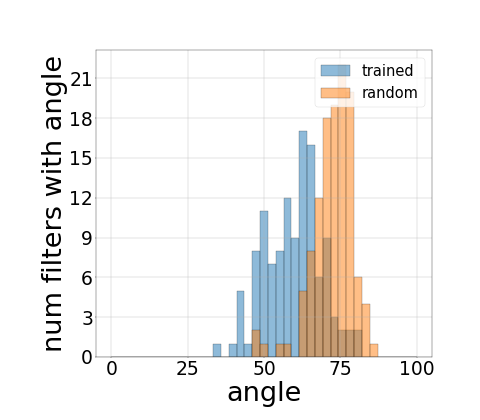}
		\caption{}
		\label{fig:exp_mnist7x7_1}
	\end{subfigure}%
	\begin{subfigure}{.5\textwidth}
		\centering
		\includegraphics[width=0.8\linewidth]{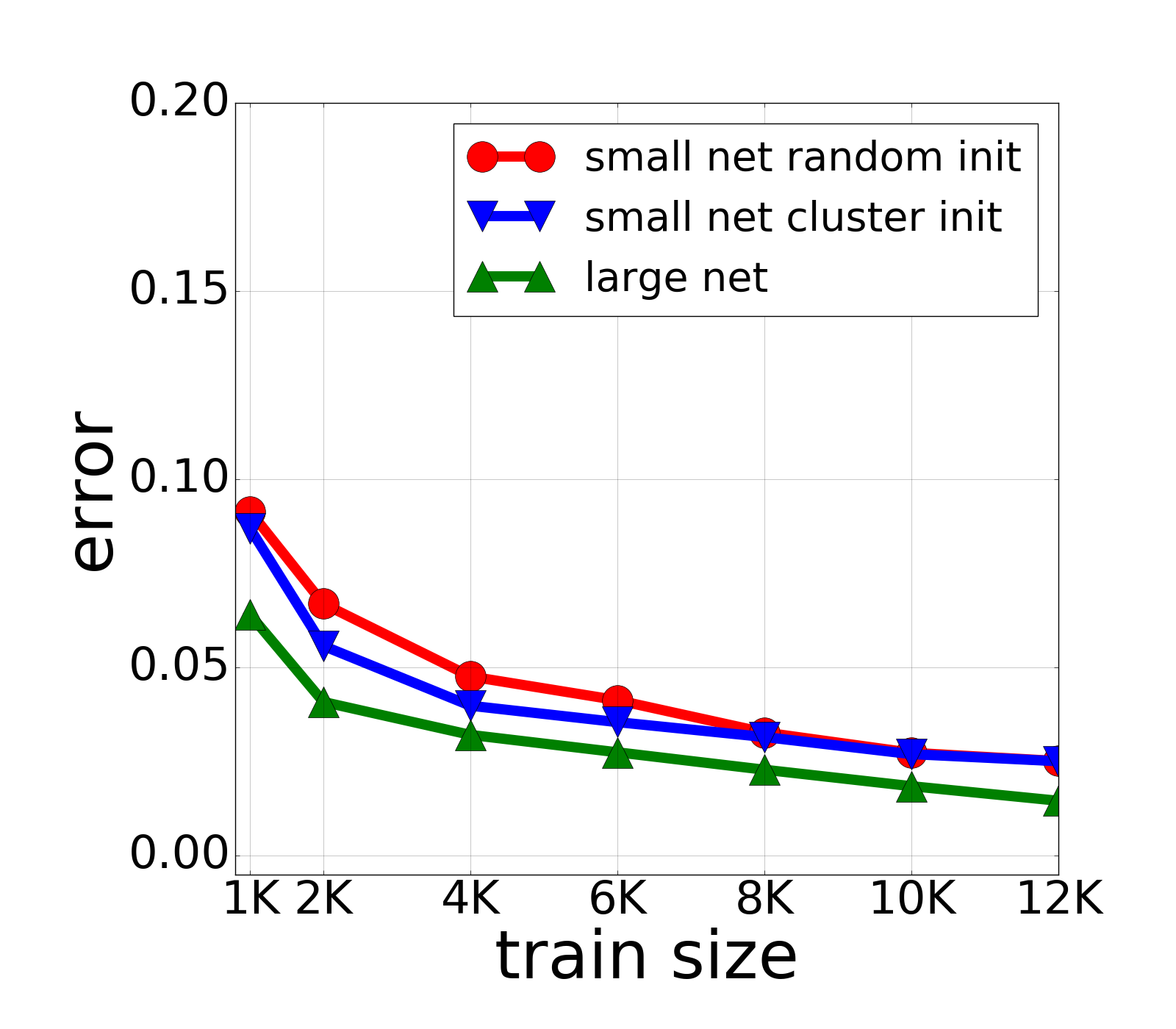}
		\caption{}
		\label{fig:exp_mnist7x7_2}
	\end{subfigure}%
	\caption{\small{ Clustering and Exploration in MNIST with 7x7 filters (a) Distribution of angle to closest center in trained and random networks. (b) The plot shows the test error of the small network (4 channels) with standard training (red), the small network that uses clusters from the large network (blue), and the large network (120 channels) with standard training (green).}}
	\label{fig:exp_mnist7x7}
\end{figure*}

\end{document}